%% file: main.tex
\documentclass[11pt]{article}
\pdfoutput=1
\usepackage[round,compress]{natbib}
\usepackage{ss} 
\include{header}
\include{macros}

\begin{document}

\title{Cost-Driven Representation Learning for \\Linear
Quadratic Gaussian Control: Part II}
\author{%
    \name Yi Tian \email{yitian@mit.edu}\\
    \addr{Massachusetts Institute of Technology}\\
    \name Kaiqing Zhang \email{kaiqing@umd.edu}\\
    \addr{University of Maryland, College Park}\\
    \name Russ Tedrake \email{russt@mit.edu}\\
    \addr{Massachusetts Institute of Technology}\\
    \name Suvrit Sra \email{s.sra@tum.de}\\
    \addr{Technical University Munich}
}
\maketitle

\input{lqg-inf}

\end{document}

%% file: header.tex
\usepackage[utf8]{inputenc}
\usepackage[T1]{fontenc}

\usepackage[verbose=true,
letterpaper,
textheight=8.7in,
textwidth=6.4in,
headheight=12pt,
headsep=25pt,
footskip=30pt]{geometry}

\usepackage{mathtools}
\usepackage{booktabs}
\usepackage{graphicx}
\usepackage{amsmath,amssymb,amsfonts,amsthm,amsxtra,bm}
\usepackage{amsthm}
\usepackage{xspace}
\usepackage{enumerate}
\usepackage{hyperref}
\usepackage{url}
\usepackage{colortbl}
\usepackage{makecell,multirow}       
\usepackage{nicefrac}       
\usepackage{thmtools}  
\usepackage{xspace}
\usepackage{footnote, tablefootnote}
\usepackage{float}
\floatstyle{plaintop}
\restylefloat{table}
\usepackage{epstopdf}
\usepackage{latexsym}
\usepackage{algorithm, algorithmicx, algpseudocode}
\algnewcommand{\LineComment}[1]{\Statex \texttt{// #1}}
\usepackage{todonotes}
\usepackage{thmtools, thm-restate}
\usepackage{bbm}

\numberwithin{equation}{section}

%% file: macros.tex
\newtheorem{theorem}{Theorem}
\newtheorem{lemma}{Lemma}

\newtheorem{proposition}{Proposition}

\newtheorem{definition}{Definition}
\newtheorem{assumption}{Assumption}

\newcommand{\corele}{\textsc{CoReL-E}}
\newcommand{\coreli}{\textsc{CoReL-I}}
\newcommand{\sysid}{\textsc{SysId}}
\newcommand{\cosysid}{\textsc{CoSysId}}

\newcommand{\raw}{\textnormal{raw}}
\newcommand{\state}{\textnormal{state}}

\newcommand{\E}{\mathbb{E}}

\newcommand{\N}{\mathbb{N}}
\newcommand{\p}{\mathbb{P}}
\newcommand{\R}{\mathbb{R}}
\newcommand{\s}{\mathbb{S}}

\newcommand{\tr}{\mathrm{tr}}
\newcommand{\Var}{\mathrm{Var}}
\newcommand{\Cov}{\mathrm{Cov}}
\DeclareMathOperator{\sign}{sign}

\DeclareMathOperator*{\argmin}{argmin}

\newcommand\ip[2]{\left\langle #1, #2 \right\rangle}
\newcommand{\ipc}[2]{\bigl\langle #1, #2 \bigr\rangle}

\newcommand{\given}{\;\vert\;}

\newcommand{\svec}{\mathrm{svec}}
\newcommand{\rank}{\mathrm{rank}}
\newcommand{\poly}{\mathrm{poly}}

\newcommand{\nlsum}{\sum\nolimits}

\newcommand{\nlmax}{\max\nolimits}

\newcommand{\bigo}{\mathcal{O}}

\newcommand{\gauss}{\mathcal{N}}

\newcommand{\eps}{\varepsilon}

\newcommand{\cD}{\mathcal{D}}

\newcommand{\cF}{\mathcal{F}}
\newcommand{\cG}{\mathcal{G}}

\newcommand{\cI}{\mathcal{I}}

\newcommand{\odelta}{\overline{\delta}}

\newcommand{\ob}{\overline{b}}
\newcommand{\oc}{\overline{c}}
\newcommand{\overe}{\overline{e}}
\newcommand{\of}{\overline{f}}

\newcommand{\oA}{\overline{A}}
\newcommand{\oB}{\overline{B}}

\newcommand{\oF}{\overline{F}}
\newcommand{\oH}{\overline{H}}
\newcommand{\oQ}{\overline{Q}}

\newcommand{\spi}{\pi^{\ast}}
\newcommand{\sSigma}{\Sigma^{\ast}}

\newcommand{\sA}{A^{\ast}}
\newcommand{\sB}{B^{\ast}}
\newcommand{\sC}{C^{\ast}}

\newcommand{\sF}{F^{\ast}}

\newcommand{\sK}{K^{\ast}}
\newcommand{\sL}{L^{\ast}}
\newcommand{\sM}{M^{\ast}}
\newcommand{\sN}{N^{\ast}}
\newcommand{\sP}{P^{\ast}}
\newcommand{\sQ}{Q^{\ast}}
\newcommand{\sR}{R^{\ast}}
\newcommand{\sS}{S^{\ast}}

\newcommand{\sX}{X^{\ast}}

\newcommand{\sob}{\ob^{\ast}}
\newcommand{\sof}{\of^{\ast}}
\newcommand{\soA}{\oA^{\ast}}
\newcommand{\soB}{\oB^{\ast}}
\newcommand{\soF}{\oF^{\ast}}
\newcommand{\soQ}{\oQ^{\ast}}
\newcommand{\starb}{b^{\ast}}

\newcommand{\starf}{f^{\ast}}
\newcommand{\sh}{h^{\ast}}

\newcommand{\sx}{x^{\ast}}
\newcommand{\sy}{y^{\ast}}
\newcommand{\sz}{z^{\ast}}

\newcommand{\epi}{\hat{\pi}}

\newcommand{\eb}{\hat{b}}

\newcommand{\ez}{\hat{z}}
\newcommand{\eA}{\hat{A}}
\newcommand{\eB}{\hat{B}}

\newcommand{\eK}{\hat{K}}

\newcommand{\eM}{\hat{M}}
\newcommand{\eN}{\hat{N}}

\newcommand{\eQ}{\hat{Q}}

\newcommand{\eS}{\hat{S}}

\newcommand{\tA}{\widetilde{A}}
\newcommand{\tB}{\widetilde{B}}

\newcommand{\tK}{\widetilde{K}}
\newcommand{\tM}{\widetilde{M}}

\newcommand{\tQ}{\widetilde{Q}}

\newcommand{\spA}{A^{\ast\prime}}
\newcommand{\spB}{B^{\ast\prime}}
\newcommand{\spC}{C^{\ast\prime}}

\newcommand{\spM}{M^{\ast\prime}}

\newcommand{\spQ}{Q^{\ast\prime}}

\makeatletter
\newcommand{\vast}{\bBigg@{3}}
\newcommand{\Vast}{\bBigg@{3.5}}
\makeatother

%% file: lqg-inf.tex

\begin{abstract}
    We study the problem of state representation learning for control from partial and potentially high-dimensional observations. We approach this problem via \emph{cost-driven state representation learning}, in which we learn a dynamical model in a latent state space by \emph{predicting cumulative costs}. 
    In particular, we establish \emph{finite-sample guarantees} on finding a near-optimal representation function and a near-optimal controller using the learned latent model for infinite-horizon time-invariant Linear Quadratic Gaussian (LQG) control. We study two approaches to cost-driven representation learning, which differ in whether the transition function of the latent state is learned explicitly or implicitly. The first approach has also been investigated in Part I of this work, for finite-horizon time-varying LQG control.  
    The second approach closely resembles \emph{MuZero}, a recent breakthrough in empirical reinforcement learning, 
    in that it learns latent dynamics \emph{implicitly} by predicting \emph{cumulative costs}. 
    A key technical contribution of this Part II is to prove persistency of excitation for a new stochastic process that arises from the analysis of quadratic regression in our approach, and may be of independent interest.

\end{abstract}

\section{Introduction}
\label{sec:intro}

Control with a \emph{learned} latent model has achieved state-of-the-art performance in several reinforcement learning (RL) benchmarks, including board games, Atari games, and visuomotor control~\citep{schrittwieser2020mastering, ye2021mastering, hafner2023mastering}. 
To better understand this machinery in RL, we introduce it to a classical optimal control problem, namely the linear quadratic Gaussian (LQG) control, and study its theoretical, in particular,  finite-sample performance. 
Essential to this approach is {the learning of} two components: a \emph{state representation} function that maps an observed history to some latent state, and a \emph{latent model} that predicts the transition and cost in the latent state space. The latent model is usually a Markov decision process, using which we obtain a policy in the latent space or execute online planning. 

What is the correct {\it objective} to optimize for learning a  good latent model? 
One popular choice is to learn a function that \emph{reconstructs the observation} from the latent state~\citep{hafner2019dream, hafner2019learning, hafner2020mastering, hafner2023mastering}. 
A latent model learned this way is \emph{agnostic to control tasks} and retains all the information about the environment. This class of approaches may achieve satisfactory performance empirically, but are prone to background distraction and control-irrelevant information  \citep{fu2021learning}. 
The second class of methods learn an \emph{inverse model} that infers actions from latent states at different time steps~\citep{pathak2017curiosity, lamb2022guaranteed}. A latent model learned with this methodology is also task-agnostic but can extract \emph{control-relevant} information. 
In contrast, the third class of methods learn \emph{task-relevant} representations by \emph{predicting costs} in the control task~\citep{oh2017value, zhang2020learning,schrittwieser2020mastering}. The concept that a good latent state should be able to predict costs is intuitive, as the costs are directly relevant to optimal control. This class of methods is the focus of this work, which aims to examine the soundness of this methodology in classical partially observable control problems, e.g., the LQG control.

In Part I of this work~\citep{tian2022cost}, we have studied provable cost-driven state representation learning in LQG for the \emph{finite-horizon, time-varying} setting. In this Part II, we build upon it and complement it by studying the same question for the \emph{infinite-horizon, time-invariant} setting. In this setting, both the representation function and the latent model are \emph{stationary}, which is usually the case in empirical RL practice. This allows us to formulate a new approach that draws an even closer connection to the state representation learning in MuZero~\citep{schrittwieser2020mastering}, an RL algorithm that matches the superhuman performance of AlphaZero in Go, shogi and chess, while outperforming model-free RL algorithms in Atari games.  

We summarize our contributions as follows. 
\begin{itemize}
    \item We show that two cost-driven state representation learning methods provably solve infinite-horizon time-invariant LQG control, with  finite-sample guarantees. 
    Both methods only need a single trajectory; one resembles the method in Part I of this work, and the other resembles the state representation learning in MuZero.
    \item By analyzing the MuZero-style algorithm, we notice the potential issue of \emph{coordinate misalignment}: Costs can be invariant to orthogonal  
    transformations of the latent states, and implicit dynamics learning by predicting {\it one-step} transition may not recover the latent state coordinates consistently.
    This insight suggests the need to predict {\it multi-step} latent transition or other coordinate alignment procedures in the MuZero-style, implicit dynamics learning approaches. 
    \item Technically, we overcome the difficulty of having \emph{correlated} data in a single trajectory for latent model learning, as we are dealing with the time-invariant setting and need to aggregate samples across time steps in contrast to the Part I of this work. To achieve so, on one hand, we prove a new result about the persistency of excitation for a stochastic process that arises from the analysis of the quadratic regression subroutine in both of our methods; on the other hand, to prove concentration beyond martingale difference sequences, we build on the idea that widely separated sample points in a mixing process are \emph{almost} independent,  and introduce a new analysis method by partitioning the sequence and applying the Gram-Schmidt process.  
\end{itemize}

\noindent \textbf{Notation.}
The notation in this Part II is the same as that in Part I of this work. For a collection of $d$-dimensional vectors $(v_t)_{t=i}^{j}$, we define $v_{i:j} := [v_i; v_{i+1}; \ldots; v_j] \in \R^{d(j-i+1)}$ as the concatenation along the column in Part I; in Part II we additionally let $v_{j:i} := [v_j; v_{j-1}; \ldots; v_i] \in \R^{d(j-i+1)}$ denote the concatenation along the column in the reverse order.
Besides, for a square matrix $A$, let $\rho(A)$ denote its spectral radius, and define $\alpha(A) := \sup_{k \ge 0} \|A^k\|_2 \rho(A)^{-k} > 1$.
Let $\s^{d}$ denote the unit sphere in $\R^{d+1}$.

\section{Problem setup}
\label{sec:setup}

A partially observable linear time-invariant (LTI) dynamical system is described by 
\begin{align}  \label{eqn:polti}
    x_{t+1} = \sA x_t + \sB u_t + w_t, \quad y_t = \sC x_t + v_t,
\end{align}
with state $x_t \in \R^{d_x}$, observation $y_t \in \R^{d_y}$, and control $u_t \in \R^{d_u}$  for all $t \ge 0$. Process noises $(w_t)_{t\ge 0}$ and observation noises $(v_t)_{t\ge 0}$ are i.i.d. sampled from $\gauss(0, \Sigma_{w_t})$ and $\gauss(0, \Sigma_{v_t})$, respectively. 
Let initial state $x_0$ be independently sampled from $\gauss(0, \Sigma_0)$.
The quadratic cost function is \mbox{given by} 
\begin{align}  \label{eqn:cost}
    c(x, u) = \|x\|_{\sQ}^2  + \|u\|_{\sR}^2,
\end{align} 
where $\sQ \succcurlyeq 0$ and $\sR \succ 0$.

A policy/controller $\pi$ determines an action/control input $u_t$ at time step $t$ based on the history $[y_{0:t}; u_{0:(t-1)}]$ up to this time step. 
For $t \ge 0$,  let $c_t := c(x_t, u_t)$ denote the cost at time step $t$. 
Given a policy $\pi$, let 
\begin{align}  \label{eqn:avg-cost-lqg}
    J(\pi) := 
    \limsup_{T\to \infty} \frac{1}{T} \nlsum_{t=0}^{T-1} \E[c_t]
\end{align}
denote the infinite-horizon time-averaged expected cost.
The goal of LQG control is to find a policy $\pi$ that minimizes $J(\pi)$.

We make the following standard assumptions.
\begin{assumption}  \label{asp:asp}
    System dynamics~\eqref{eqn:polti} and \mbox{cost~\eqref{eqn:cost} satisfy:}
    \begin{enumerate}
        \item \label{asp:stb} The system is stable, that is, $\rho(\sA) < 1$. 
        \item \label{asp:ab-ctrl} $(\sA, \sB)$ is $\nu$-controllable for some $\nu>0$, that is, the controllability matrix
        \begin{align*}
            \Phi_c(\sA, \sB) := [\sB, \sA\sB, \ldots, (\sA)^{d_x-1} \sB]
        \end{align*}
        has rank $d_x$ and $\sigma_{\min}(\Phi_c(\sA, \sB)) \ge \nu$. 
        \item \label{asp:ac-obs} $(\sA, \sC)$ is $\omega$-observable for some $\omega > 0$, that is, the observability matrix 
        \begin{align*}
            \Phi_o(\sA, \sC) := [\sC; \sC \sA; \ldots; \sC (\sA)^{d_x - 1}]
        \end{align*}
        has rank $d_x$ and $\sigma_{\min}(\Phi_o(\sA, \sC)) \ge \omega$.
        \item \label{asp:aw-ctrl} $(\sA, \Sigma_w^{1/2})$ is $\kappa$-controllable for some $\kappa > 0$.
        \item \label{asp:aq-obs} $(\sA, (\sQ)^{1/2})$ is $\mu$-observable for some $\mu > 0$.
        \item \label{asp:v} $\Sigma_v \succcurlyeq \sigma_v^2 I$ for some $ \sigma_v > 0$; this can always be achieved by inserting  Gaussian noises with full-rank covariance matrices to the observations.
        \item \label{asp:r} $\sR \succcurlyeq r^2 I$ for some $r > 0$.
        \item \label{asp:reg} The operator norms of $\sA$, $\sB$, $\sC$, $\sQ$, $\sR$, $\Sigma_w$, $\Sigma_v$, $\Sigma_0$ and $\alpha(\sA), \alpha(\soA)$ are $\bigo(1)$,
        the singular value lower bounds $\nu$, $\omega$, $\nu$, $\kappa$, $\sigma_v$, $r$ and spectral radii $\rho(\sA), \rho(\soA)$ are $\Omega(1)$, where $\soA:=(I - \sL \sC) \sA$ with $\sL$ defined in \eqref{eqn:L-riccati}.  
    \end{enumerate}
\end{assumption}

If the system parameters $(\sA, \sB, \sC, \sQ, \sR, \Sigma_w, \Sigma_v)$ are known, the optimal policy is obtained by combining the Kalman filter 
\begin{align} 
    \sz_{t+1} = \sA \sz_t + \sB u_t + \sL ( y_{t+1} - \sC (\sA \sz_t + \sB u_t) )  \label{eqn:kf}
\end{align}
with the optimal feedback gain $\sK$ of the linear quadratic regulator such that $u_t = \sK \sz_t$, where $\sL$ is the Kalman gain, and at the initial time step, we can set, e.g., $\sz_0 = \sL y_0$. This fact is known as the \emph{separation principle}, and the Kalman gain and optimal feedback gain are \mbox{given by} 
\begin{align}
    \sL =\;& \sS (\sC)^{\top} (\sC \sS (\sC)^{\top} + \Sigma_v)^{-1},\label{eqn:L-riccati} \\
    \sK =\;& -((\sB)^{\top} \sP \sB + R)^{-1} (\sB)^{\top} \sP \sA,\label{eqn:K-riccati}
\end{align}
where $\sS$ and $\sP$ are determined by their respective discrete-time algebraic Riccati \mbox{equations (DAREs)}:
\begin{gather}
    \begin{aligned}
    \sS = \sA \big(\sS - \sS (\sC)^{\top} (&\sC \sS (\sC)^{\top} + \Sigma_v)^{-1} \sC \sS \big) (\sA)^{\top} + \Sigma_w,
    \end{aligned}
    \label{eqn:s-riccati} \\
    \begin{aligned}
    \sP = (\sA)^{\top} \big(\sP - \sP \sB (&(\sB)^{\top} \sP \sB + \sR)^{-1} (\sB)^{\top} \sP \big) \sA + \sQ. 
    \end{aligned}
    \label{eqn:p-riccati}
\end{gather}

Assumptions~\ref{asp:asp}.\ref{asp:ab-ctrl} to~\ref{asp:asp}.\ref{asp:r} guarantee the existence and uniqueness of the positive definite solutions $\sS$ and $\sP$; Assumption~\ref{asp:asp}.\ref{asp:reg} further guarantees that their operator norms are $\bigo(1)$ and minimum singular values are $\Omega(1)$. 
Hence, $\|\sL\|_2$ and $\|\sK\|_2$ are of order $\bigo(1)$.
The assumption on $\alpha(\sA), \alpha(\soA), \rho(\sA), \rho(\soA)$ provides guarantees for state estimation from a finite history and has also been made in the literature~\citep{mania2019certainty, oymak2019non}. If $\rho(\sA)$ or $\rho(\soA)$ equals zero, then $(\sA)^{d_x}$ or $(\soA)^{d_x}$ is a zero matrix by the Cayley-Hamilton theorem, so using history length $H \ge d_x$ completely eliminates the truncation errors. Thus, Assumption~\ref{asp:asp}.\ref{asp:reg} does not lose generality. Let $\alpha := \max(\alpha(\sA), \alpha(\soA))$ and $\rho := \max(\rho(\sA), \rho(\soA))$.

We consider the data-driven control setting, where the LQG model $(\sA, \sB, \sC, \sQ, \Sigma_w, \Sigma_v)$ is unknown. For simplicity, we assume $\sR$ is known, though our approaches can be readily extended to the case where it is unknown, which we discuss in more detail in~\S\ref{sec:repl}. 

\subsection{Latent model of infinite-horizon time-invariant LQG}  
\label{sec:latent-model}

The stationary Kalman filter~\eqref{eqn:kf} asymptotically produces the optimal \emph{state estimation}  in the sense of minimum mean squared errors. With a finite horizon, however, the optimal state estimator is time-varying, given by 
\begin{align}  \label{eqn:tv-kf}
    \sz_{t+1} = \sA \sz_t + \sB u_t + \sL_{t+1} ( y_{t+1} - \sC (\sA \sz_t + \sB u_t) ),
\end{align}
where $\sL_t$ is the time-varying Kalman gain, converging to $\sL$ as $t \to \infty$. 
This convergence is equivalent to that of the error covariance matrix $\E[(x_t - \sz_t) (x_t - \sz_t)^{\top}]$, which is exponentially fast~\citep{komaroff1994iterative}.
Hence, for simplicity, we assume this error covariance matrix is stationary at the initial time step by the choice of $\sz_0$ so that $\sL_t = \sL$ for $t \ge 1$; this assumption has also been adopted in the literature~\citep{lale2020logarithmic, lale2021adaptive,jadbabaie2021time}.

The \emph{innovation} term $i_{t+1} := y_{t+1} - \sC(\sA \sz_t + \sB u_t)$ is independent of $\sz_0$ and the history $(u_{0}, y_{1}, \ldots, u_{t-1}, y_{t})$ and $(i_t)_{t\ge 1}$ are mutually independent.
The following proposition, taken from~\citep[Proposition 1]{tian2022cost}, represents the system in terms of the state estimates obtained by the Kalman filter, which we refer to as the \emph{latent model}.

\begin{proposition}  \label{prp:latent-model}
    Let $\sz_0$ be the initial state estimate 
    and $(\sz_t)_{t\ge 1}$ be the state estimates given by the Kalman filter. Then, for $t \ge 0$, 
    \begin{align*}
        \sz_{t+1} = \sA \sz_t + \sB u_t + \sL i_{t+1},
    \end{align*}
    where $\sL i_{t+1}$ is independent of $\sz_t$ and $u_t$, i.e., the state estimates follow the same linear dynamics with noises $\{\sL i_{t+1}\}_{t\geq 0}$. 
    The cost at step $t$ can be reformulated as a function of the state estimates by 
    \begin{align*}
        c_t = \|\sz_t\|_{\sQ}^2 + \|u_t\|_{\sR}^2 + \starb + \gamma_t + \eta_t,
    \end{align*}
    where $\starb = \E[\|x_t - \sz_t\|_{\sQ}^2] > 0$, and $\gamma_t = \|x_t - \sz_t\|_{\sQ}^2 - \starb$, $\eta_t = 2\ipc{\sz_t}{x_t - \sz_t}_{\sQ}$ are both zero-mean subexponential random variables. Moreover, $\starb = \bigo(1)$ and $\|\gamma_t\|_{\psi_1} = \bigo(d_x^{1/2})$; if control $u_t \sim \gauss(0, \sigma_u^2 I)$ for $t \ge 0$, then we have $\|\eta_t\|_{\psi_1} = \bigo(d_x^{1/2})$.
\end{proposition}

Proposition~\ref{prp:latent-model} shows that the dynamics of the state estimates computed by the time-varying Kalman filter are the same as the original system up to noises; the costs are also the same, up to constants and noises. Hence, a latent model can be parameterized by $(A, B, Q, \sR)$, with the constant $\starb$ and noises neglected due to their irrelevance to planning.
A stationary latent policy is a linear controller $u_t = K z_t$ on latent state $z_t$, parameterized by the feedback \mbox{gain $K\in\R^{d_u \times d_x}$}.

The latent model enables us to find a good latent policy. 
To learn such a latent model and to deploy a latent policy in the original partially observable system, we need a \emph{representation function}.
Let $\soA := (I - \sL \sC) \sA$ and $\soB := (I - \sL \sC) \sB$. 
Then, the Kalman filter can be written as $\sz_{t+1} = \soA \sz_t + \soB u_t + \sL y_{t+1}$. For $t \ge 0$, unrolling the recursion gives
\begin{align*}
    \sz_t &= \soA (\soA \sz_{t-2} + \soB u_{t-2} + \sL y_{t-1}) + \soB u_{t-1} + \sL y_t \\
    &= [(\soA)^{t-1} \sL, \ldots, \sL] y_{1:t} + [(\soA)^{t-1}\soB, \ldots, \soB] u_{0:(t-1)} + (\soA)^{t} \sz_0 \\
    &=: \sM_t [y_{1:t}; u_{0:(t-1)}; \sz_0],
\end{align*}
where $\sM_t \in \R^{d_x \times (t d_y + t d_u + d_x)}$. This means that the representation function can be parameterized as \emph{linear}  mappings for full histories (with $y_0$ replaced by $\sz_0$). 

Despite the simplicity, the input dimension of the function grows linearly in time, making it intractable to estimate the state using the full history for large $t$; nor it is  necessary, since the impact of old data decreases exponentially. Under Assumption~\ref{asp:asp}, $\rho(\soA) < 1$~\citep[Appendix E.4]{bertsekas2012dynamic}.
With an $H$-step truncated history, the state estimate can thus be written as 
\begin{align}
    \sz_t &= [(\soA)^{H-1} \sL, \ldots, \sL] y_{(t-H+1):t} + [(\soA)^{H-1} \soB, \ldots, \soB] u_{(t-H):(t-1)} + \delta_{t} \notag\\
    &=: \sM [y_{(t-H+1):t}; u_{(t-H):(t-1)}] + \delta_t,\label{equ:z_s_relation}
\end{align}
where $\delta_t = (\soA)^H \sz_{t-H}$ denotes the truncation error, whose impact decays exponentially in $H$ and can be neglected for sufficiently large $H$, since $\sz_{t-H}$ converges to a stationary distribution and its norm is bounded with high probability.
Hence, the representation function that we aim to recover is $\sM \in \R^{d_x \times H(d_y + d_u)}$, which has an $\bigo(1)$ operate norm and takes as input the $H$-step history $h_t = [y_{(t-H+1):t}; u_{(t-H):(t-1)}]$. Henceforth, we let $d_h := H (d_y + d_u)$.
Then, a representation function is parameterized by a \mbox{matrix $M \in \R^{d_x \times d_h}$}.

Overall, a policy is a combination of a state representation function parameterized by $M$ and a feedback gain $K$ in the latent model, denoted by $\pi = (M, K)$. Such a policy can be applied after we have $H$ steps of history; for the first $H$ steps, we can use an arbitrary stabilizing policy, e.g., zero or zero-mean Gaussian control inputs. 
Learning to solve LQG control in this framework can thus be achieved by: 1) learning the state representation function parameter $M$; 2) extracting latent model $(A, B, Q, \sR)$; and 3) finding the optimal $K$ by planning in the latent model. 
Note that policy $(\sM, \sK)$ is near, but not exactly, optimal due to the truncation error $\delta_t$; the exactly optimal policy is still characterized by $(\sL, \sK)$. 
Next, we introduce our approach following \mbox{this pipeline}.

\section{Method}

In practice, latent model learning methods collect trajectories by interacting with the system online  using some  policy; 
the trajectories are used to improve the learned latent model, which in turn improves the policy.
In LQG control, it is known that 
one can learn a good latent model from a single trajectory, collected using zero-mean Gaussian control inputs, by viewing this procedure as a classical \emph{system identification}  problem; see e.g.,~\citep{oymak2019non}. 
This is also how we assume the data are collected. We note that our results also apply to data from multiple independent trajectories using control inputs from the same zero-mean \mbox{Gaussian distribution}.

\begin{algorithm}[!t]
    \caption{Cost-driven state representation learning} 
    \label{alg:lmbc}
    \begin{algorithmic}[1]
        \State {\bfseries Input:} length $T$, history length $H$, noise magnitude $\sigma_u$
        \State Collect trajectories of length $T+H$ using $u_t \sim \gauss(0, \sigma_u^2 I)$, for $t \ge 0$, to obtain
        \begin{equation}  \label{eqn:traj}
            \begin{aligned}  
                \cD_{\raw} =\;& (y_0, u_0, c_0, y_1, u_1, c_1, \ldots, y_{T+H-1}, u_{T+H-1}, c_{T+H-1}, y_{T+H})
            \end{aligned}
        \end{equation}
        \State Estimate the state representation function and cost constants by solving
        \begin{align}
            \eN, \eb_0 \in\;& \argmin_{N = N^{\top}, b_0}~~ \nlsum_{t=H}^{T+H-1} \big( \big\|h_t \big\|_{N}^2 + b_0 - \oc_t \big)^2,  \label{eqn:alg-qr}
        \end{align}
        where $h_t = [y_{(t-H+1):t}; u_{(t-H):(t-1)}]$ and $\oc_t := \sum_{\tau=t}^{t+d_x-1} (c_{\tau} - \|u_{\tau}\|_{\sR}^2)$
        \State Find $\eM \in \argmin_{M \in \R^{d_x \times H(d_y + d_u)}} \| M^{\top} M - \eN\|_{F}$
        \State Compute $\ez_t = \eM [y_{(t-H+1):t}; u_{(t-H):(t-1)}]$ for all $t \ge H$, so that the data are converted to $\cD_{\state}$:
        \begin{align*}
            (\ez_H, u_H, c_H, \ldots, \ez_{T+H-1}, u_{T+H-1}, c_{T+H-1}, \ez_{T+H})
        \end{align*}
        \State Run \sysid{}~\eqref{eqn:sys-id} or \cosysid{} (Algorithm~\ref{alg:cosysid}) to obtain system dynamics  matrices $(\eA, \eB)$
        \State Estimate the cost function by solving 
        \begin{align}  \label{eqn:alg-qr-q}
            \tQ, \eb \in \argmin_{Q = Q^{\top}, b}~~ \nlsum_{t=H}^{T+H-1} (\|\ez_t\|_Q^2 + \|u_t\|_{\sR}^2 + b - c_t)^2
        \end{align}
        \State Truncate negative eigenvalues of $\tQ$ to $0$ to obtain $\eQ \succcurlyeq 0$
        \State Find feedback gain $\eK$ from $(\eA, \eB, \eQ, \sR)$ by solving DARE~\eqref{eqn:p-riccati} and~\eqref{eqn:K-riccati} 
        \State {\bfseries Return:} policy $\epi = (\eM, \eK)$
    \end{algorithmic}
\end{algorithm}

In our cost-driven state representation learning approach, state representations are learned by predicting costs. 
To learn the transition function in the latent model, two approaches have been explored in the literature. The first approach explicitly minimizes the transition prediction error~\citep{subramanian2020approximate, hafner2019dream}. Algorithmically, the overall loss is a combination of cost prediction and transition prediction errors.

The second approach, as taken by MuZero in~\citep{schrittwieser2020mastering}, learns the transition dynamics \emph{implicitly}, by minimizing {\it cost prediction errors} at {\it future states}  generated from the transition function~\citep{oh2017value, schrittwieser2020mastering}.
Algorithmically, the overall loss aggregates the cost prediction errors \emph{across multiple time steps}.
In both approaches, the coupling of different terms in the loss makes finite-sample analysis difficult. As observed in Part I of this work, the structure of LQG allows us to learn the representation function independently of learning the transition function. This allows us to formulate both approaches under the same cost-driven state representation learning framework \mbox{(Algorithm~\ref{alg:lmbc}).}

Algorithm~\ref{alg:lmbc} consists of three main steps. Lines 3 to 5 correspond to cost-driven representation function learning. 
Lines 6 to 8 correspond to latent model learning, where the system dynamics can be identified either explicitly, by ordinary least squares (\sysid{}), or implicitly, by future cost prediction (\cosysid{}, Algorithm~\ref{alg:cosysid}). 
Line 9 corresponds to the policy optimization procedure in the latent model; in LQG this amounts to solving DAREs. Below we elaborate on cost-driven representation learning, \sysid{}, and \cosysid{} in order.  

\subsection{Cost-driven representation function learning}\label{sec:repl}

The procedure of cost-driven representation function learning is consistent with Part I of this work.
The main idea is to perform quadratic regression~\eqref{eqn:alg-qr} to the $d_x$-step cumulative costs; this step corresponds to the value prediction in MuZero~\citep{schrittwieser2020mastering}. 
By the $\mu$-observability of $(\sA, (\sQ)^{1/2})$ (Assumption~\ref{asp:asp}.\ref{asp:aq-obs}), the cost observability Gram matrix satisfies 
\begin{align*}
    \soQ := \nlsum_{t=0}^{d_x-1} ((\sA)^{t})^{\top} \sQ (\sA)^{t} \succcurlyeq \mu^2 I.
\end{align*}
Under zero control and zero noise, starting from $x$, the $d_x$-step cumulative cost is precisely $\|x\|_{\soQ}^2$. Hence, with the impact of zero-mean actions and zero-mean noises averaged out, $\eN$ estimates $\sN := (\sM)^{\top} \soQ \sM$; 
up to an orthogonal transformation, $\eM$ recovers $\spM := (\soQ)^{1/2} \sM$, the representation function under an equivalent parameterization, termed as the \emph{normalized parameterization} in Part I of this work, where  
\begin{gather*}
    \spA = (\soQ)^{1/2} \sA (\soQ)^{-1/2}, \quad \spB = (\soQ)^{1/2} \sB, \quad \spC = \sC (\soQ)^{-1/2}, \\
    w_t^{\prime} = (\soQ)^{1/2} w_t, \quad \spQ = (\soQ)^{-1/2} \sQ (\soQ)^{-1/2}.
\end{gather*}
Without loss of generality, we assume that system~\eqref{eqn:polti} is in the normalized parameterization. 

Note that with a known $\sR$, we subtract $\nlsum_{\tau=t}^{t+d_x-1} \|u_{\tau}\|_{\sR}^{2}$ from $\oc_t$ in~\eqref{eqn:alg-qr} to reduce its variance for the benefit of regression. However, the subtraction is not necessary if $\sR$ is unknown, as Proposition~\ref{prp:multi-step-cost} still holds without it. In this case, we can learn $\sR$ subsequently in~\eqref{eqn:alg-qr-q}. 

Due to the following proposition, the algorithm does not need to know the dimension $d_x$ of the latent model; it can discover $d_x$ from the eigenvalues of $\eN$.

\begin{proposition}  \label{prp:full-rank-cov}
    Under i.i.d. actions $u_t \sim \gauss(0, \sigma_u^2 I)$ for $t \ge 0$, $\lambda_{\min}(\Cov(z_t^{\ast})) = \Omega(\nu^2)$ for $t \ge d_x$, where $\nu$ is defined in Assumption~\ref{asp:asp}.\ref{asp:ab-ctrl}. 
    Recall that for a square matrix $A$, we define $\alpha(A) := \sup_{k \ge 0} \|A^k\|_2 \rho(A)^{-k}$.
    There exists a dimension-free constant $a > 0$, such that as long as $H \ge \frac{\log (a \alpha(\soA))}{\log (\rho(\soA)^{-1})}$, $\sM$ has rank $d_x$ and $\sigma_{\min}(\sM) = \Omega(\nu H^{-1/2})$.
\end{proposition}

\begin{proof}
    For $t \ge d_x$, unrolling the Kalman filter gives  
    \begin{align*}
        \sz_t =\;& \sA \sz_{t-1} + \sB u_{t-1} + \sL i_t \\
        =\;& \sA (\sA \sz_{t-2} + \sB u_{t-2} + \sL i_{t-1}) + \sL i_t \\
        =\;& [\sB, \ldots, (\sA)^{d_x-1} \sB][u_{t-1}; \ldots; u_{t-d_x}] + (\sA)^{d_x} \sz_{t-d_x} + [\sL, \ldots, (\sA)^{d_x-1} \sL][i_t; \ldots; i_{t-d_x+1}],
    \end{align*}
    where $(u_{\tau})_{\tau=t-d_x}^{t-1}$, $\sz_{t-d_x}$ and $(i_{\tau})_{\tau=t-d_x+1}^{t}$ are independent. 
    For $H \ge d_x$, the matrix multiplied by $[u_{t-1}; \ldots; u_{t-d_x}]$ is precisely the controllability matrix $\Phi_{c}(\sA, \sB)$. Then, 
    \begin{align*}
        \Cov(\sz_t)
        = \E[\sz_t (\sz_t)^{\top}]
        \succcurlyeq\;& \Phi_{c}(\sA, \sB) \E[[u_{t-1}; \ldots; u_{t-\ell}] [u_{t-1}; \ldots; u_{t-\ell}]^{\top}]\Phi_{c}^{\top}(\sA, \sB) \\
        =\;& \sigma_u^2 \Phi_{c}(\sA, \sB) \Phi_{c}^{\top}(\sA, \sB).
    \end{align*}

    By the $\nu$-controllability of $(\sA, \sB)$, $\Cov(\sz_t)$ is full-rank and $\lambda_{\min}(\Cov(\sz_t)) \ge \sigma_u^2 \nu^2$.
    Since $\sz_t = \sM h_t + \delta_t$ by \eqref{equ:z_s_relation}, we have
    \begin{align*}
        \Cov(\sM h_t) = \Cov(\sz_t - \delta_t) = \Cov(\sz_t) + \Cov(\delta_t)  - \Cov(\sz_t, \delta_t) - \Cov(\delta_t, \sz_t).
    \end{align*}
    Then, 
    \begin{align*}
        \|\Cov(\sz_t, \delta_t)\|_2 
        = \|\Cov(\delta_t, \sz_t)\|_2
        =\;& \|\E[\sz_t \delta_t^{\top}]\|_2 \\
        \overset{(i)}{\le}\;& \|\E[\sz_t (\sz_t)^{\top}]\|_2^{1/2} \cdot \|\E[\delta_t \delta_t^{\top}]\|_2^{1/2} \\
        =\;& \|\Cov(\sz_t)\|_2^{1/2} \cdot \|\Cov(\delta_t)\|_2^{1/2},
    \end{align*}
    where $(i)$ is due to~\citep[Lemma 13]{tian2022cost}.
    Hence, by Weyl's inequality, 
    \begin{align*}
        \lambda_{\min}(\Cov(\sM h_t))
        \ge\;& \lambda_{\min}(\Cov(\sz_t)) - 2 \|\Cov(\sz_t)\|_2^{1/2} \cdot \|\Cov(\delta_t)\|_2^{1/2}.
    \end{align*}
    Since $\|\Cov(\sz_t)\|_2 = \bigo(1)$ due to the stability of $\sA$ and $\delta_t = (\soA)^H \sz_{t-H}$, there exists some dimension-free constant $a > 0$ such that as long as $H \ge \frac{\log (a \alpha(\soA))}{\log (\rho(\soA)^{-1})}$, 
    \begin{align*}
        \lambda_{\min}(\Cov(\sM h_t)) \ge \sigma_u^2 \nu^2 / 2.
    \end{align*}

    On the other hand, 
    \begin{align*}
        \E[\sM h_t h_t^{\top} (\sM)^{\top}] \preccurlyeq \|\E[h_t h_t^{\top}]\|_2 \sM (\sM)^{\top}.
    \end{align*}
    Since $h_t = [y_{(t-H+1):t}; u_{(t-H):(t-1)}]$ and $(\Cov(y_t))_{t\ge 0}$, $(\Cov(u_t))_{t\ge 0}$ have $\bigo(1)$ operator norms, by~\citep[Lemma 12]{tian2022cost}, $\|\Cov(h_t)\|_2 = \|\E[h_t h_t^{\top}]\|_2 = \bigo(H)$.
    Hence,
    \begin{align*}
        0 < \sigma_u^2 \nu^2 / 2 \le \lambda_{\min}(\Cov(\sM h_t)) = \bigo(H) \sigma_{d_x}^2(\sM).
    \end{align*}
    Since $\sM \in \R^{d_x \times d_h}$, this implies that $\rank(\sM) = d_x$ and $\sigma_{\min}(\sM) = \Omega(\nu H^{-1/2})$.
\end{proof}

Proposition~\ref{prp:full-rank-cov} is an adaptation of~\citep[Proposition 2]{tian2022cost} to the infinite-horizon LTI setting. 
Necessarily, this implies that by our choice of $H$, $d_h = H(d_y + d_u) \ge d_x$.
Moreover, since $\soQ \succcurlyeq \mu^2 I$, $\sN = (\sM)^{\top} \soQ \sM$ is a $d_h \times d_h$ matrix with rank $d_x$, and $\lambda_{\min}^{+}(\sN) \ge \lambda_{\min}(\soQ) \sigma_{\min}^2(\sM) = \Omega(\mu^2 \nu^2 H^{-1})$. Hence, if $\eN$ is sufficiently close to $\sN$, by setting an appropriate threshold on the eigenvalues of $\eN$, the dimension of the latent model equals the number of eigenvalues above it.

To find an approximate factorization of $\eN$, let $\eN = U \Lambda U^{\top}$ be its eigenvalue decomposition, where the diagonal elements of $\Lambda$ are listed in descending order, and $U$ is an orthogonal matrix. Let $\Lambda_{d_x}$ be the top-left $d_x \times d_x$ block of $\Lambda$ and $U_{d_x}$ be the left $d_x$ columns of $U$. By the Eckart-Young-Mirsky theorem, $\eM = \max(\Lambda_{d_x}, 0)^{1/2} U_{d_x}^{\top}$, where ``$\max$'' applies elementwise, is the solution to Line 4 of Algorithm~\ref{alg:lmbc}, that is, the best approximate factorization of $\eN$ among $d_x \times d_h$ matrices in terms of the Frobenius norm approximation error.

In the next two subsections, we move on to discussing the learning of latent dynamics, including the explicit approach \sysid{} and the implicit approach \cosysid{}.

\subsection{Explicit learning of system dynamics}

Explicit learning of the system dynamics simply minimizes the \emph{transition prediction error} in the latent space~\citep{subramanian2020approximate}, or more generally, the statistical distances between the predicted and estimated distributions of the next latent state, like the KL divergence~\citep{hafner2019dream}.
In linear systems, it suffices to use the ordinary least squares as the \sysid{} procedure, that is, to solve 
\begin{align}  \label{eqn:sys-id}
    (\eA, \eB) \in \argmin_{A, B}~~\nlsum_{t=H}^{T+H-1} \|A \ez_t + B u_t - \ez_{t+1} \|^2.
\end{align}

In this linear regression, if $(\ez_t)_{t\ge H}$ are the optimal state estimates $(\sz_t)_{t\ge H}$~\eqref{eqn:tv-kf}, then \citep{simchowitz2018learning} has shown finite-sample guarantees for obtaining $(\eA, \eB)$. Here, however,  $\ez_t$ contains errors resulting from the representation function $\eM$ and the residual error $\delta_t$ in \eqref{equ:z_s_relation}, but as long as $T$ and $H$ are large enough, \sysid{} still has a finite-sample guarantee, as will be shown in Lemma~\ref{lem:pert-lr}.
We refer to the algorithm that instantiates Algorithm~\ref{alg:lmbc} with \sysid{} as \corele{} (\textsc{Co}st-driven state \textsc{Re}presentation \textsc{L}earning). As the time-varying counterpart in Part I of this work, this algorithm  provably solves LQG control without model knowledge, as will be shown in Theorem~\ref{thm:main-poly}. 

\subsection{Implicit learning of system dynamics (MuZero-style)}
\label{sec:implicit-dyn}

An important ingredient of latent model learning in MuZero~\citep{schrittwieser2020mastering} is to \emph{implicitly} learn the transition function by minimizing the cost prediction error at {\it future latent states} generated from the transition function. 
Let $z_{t} = M h_{t}$ denote the latent state given by the representation function parameter $M$ at step $t$. Let $z_{t, 0} = z_{t}$ and $z_{t, i} = A z_{t, i-1} + B u_{t+i-1}$ for $i \ge 1$ be the future latent state predicted by dynamics $(A, B)$ from $z_t$ after $i$ steps of transition.
For a trajectory of length $T + H$ like~\eqref{eqn:traj}, the loss that considers $\ell$ steps into the future is given by
\begin{align*}
    \nlsum_{t=H}^{T+H-\ell-1} \nlsum_{i=0}^{\ell} (\|z_{t, i}\|_{Q}^2 + \|u_{t+i}\|_{\sR}^2 + b - c_{t+i})^2.
\end{align*}

This loss involves powers of $A$ up to $A^{\ell}$; with the squared norm, the powers double, making the minimization over $A$ hard to solve and analyze for $\ell \ge 2$.
In LQG control, as we shall discuss, it suffices to take $\ell = 1$. 
The MuZero algorithm also predicts optimal values and optimal actions; in LQG, to handle the case $\sQ \not\succ 0$, like cost-driven representation learning (see \S\ref{sec:repl}), we adopt the \emph{cumulative costs} and use the normalized parameterization.
Recall that in Algorithm~\ref{alg:lmbc} we define $\oc_t := \sum_{\tau=t}^{t+d_x-1} (c_{\tau} - \|u_{\tau}\|_{\sR}^2)$.
Then, the optimization problem we aim to solve is given by
\begin{equation}
    \begin{aligned}
        \min_{M, A, B, b}~~& \nlsum_{t=H}^{T+H-1} \big( (\|M h_t\|^2 + b - \oc_t)^2 + (\|A M h_t + B u_t\|^2 + b - \oc_{t+1})^2 \big).
    \end{aligned}  \label{eqn:muzero-v}
\end{equation}

To convexify the optimization problem~\eqref{eqn:muzero-v}, we define $N := M^{\top} M$ and $N_1 := [AM, B]^{\top} [AM, B]$. Then, \eqref{eqn:muzero-v} becomes 
\begin{equation}
    \begin{aligned}
        \min_{N, N_1, b} ~~& \nlsum_{t=H}^{T+H-1} \big( (\|h_t\|_{N}^2 + b - \oc_t)^2 + (\|[h_t; u_t]\|_{N_1}^2 + b - \oc_{t+1})^2 \big).
    \end{aligned}  \label{eqn:muzero-v-cvx}
\end{equation}
This minimization problem is convex in $N$, $N_1$, and $b$, and has a closed-form solution; essentially, it consists of two linear regression problems coupled by $b$. 
As a relaxation, we can decouple the two regression problems by allowing $b$ to take different values in them, as $b$ is a term accounting for the estimation error, not part of the representation function. 
This decoupling further simplifies the analysis: the first regression problem is exactly cost-driven representation learning (\S\ref{sec:repl}), and the second is cost-driven system identification (\cosysid{}, Algorithm~\ref{alg:cosysid}).
The algorithm that instantiates Algorithm~\ref{alg:lmbc} with \cosysid{} will be referred to as   \coreli{} (\textsc{Co}st-driven state \textsc{Re}presentation and \textsc{Dy}namic \textsc{L}earning). 
Like \corele{}, this MuZero-style latent model learning method provably solves LQG, as we will show next. 

\begin{algorithm}[!t]
    \caption{\cosysid{}: Cost-driven system identification}
    \label{alg:cosysid}
    \begin{algorithmic}[1]
        \State {\bfseries Input:} data $\cD_{\raw}$ from Algorithm \ref{alg:lmbc},  representation function parameter  $\eM$
        \State Estimate the system dynamics by 
        \begin{align}
            \eN_1, \eb_1 \in \argmin_{N_1 = N_1^{\top}, b_1} \nlsum_{t=H}^{T+H-1} \big(  \|[h_t; u_t]\|_{N_1}^2 + b_1 - \oc_{t+1} \big)^2 \label{eqn:dy-qr}
        \end{align}
        \State Find $\eM_1 \in \argmin_{M_1 \in \R^{d_x \times (H d_y + (H+1) d_u)}} \|M_1^{\top} M_1 - \eN_1\|_F$
        \State Split $\eM_1$ to $[\tM, \tB]$ at column $H(d_y+d_u)$ and set $\tA = \tM \eM^{\dagger}$.
        \State Find alignment matrix $\eS_0$ by solving 
        \begin{align}  \label{eqn:align-lr}
            \eS_0 \in \argmin_{S_0 \in \R^{d_x \times d_x}} \nlsum_{t=H}^{T+H-1} \|S_0 \eM_1 [h_t; u_t] - \eM h_{t+1}\|^2
        \end{align}
        \State {\bfseries Return:} system dynamics estimate $(\eA, \eB) = (\eS_0\tA, \eS_0 \tB)$
    \end{algorithmic}
\end{algorithm}

\cosysid{} has similar steps as cost-driven representation learning (\S\ref{sec:repl}), except that in Line 5 of Algorithm~\ref{alg:cosysid}, it requires fitting a matrix $\eS_0$.
This is because the cost is {\it invariant} to the orthogonal transformations of latent states, and the approximate factorization steps recover $\sM$ and $\sM_1$ up to orthogonal transformations $S$ and $S_1$, but there is no guarantee for the two transformations to be the same.
MuZero bypasses this problem by predicting \emph{multiple} steps of costs into the future, but analyzing such an optimization function involves the additional complexity of dealing with higher-order powers of $A$. Here, we instead estimate the $S_0 = S S_1^{\top}$ to align such two transformations. 
We note that although \cosysid{} needs the output $\eM$ from cost-driven representation learning, the two quadratic regressions~\eqref{eqn:alg-qr} and~\eqref{eqn:dy-qr} are not coupled and can be solved \mbox{in parallel}.

\vspace{2pt}
\noindent\textbf{Discussion on \cosysid{}.}
In \cosysid{} (Algorithm~\ref{alg:cosysid}), the covariates of the quadratic regression in~\eqref{eqn:dy-qr} are $([h_t; u_t])_{t\ge H}$. One may wonder if we can pursue an alternative approach by fixing $M$ to be $\eM$, and using $([\ez_t; u_t])_{t\ge H}$ as covariates, which have a much lower dimension, though the two quadratic regressions cannot be solved in parallel anymore. 

Specifically, the new quadratic regression we need to solve is given by
\begin{align*}
    \eN_2, \eb_2 \in\;& \argmin_{N_2 = N_2^{\top}, b_2}~~\nlsum_{t=H}^{T+H-1} \big( \|[\ez_t; u_t]\|_{N_2}^2 + b_2 - \oc_{t+1} \big)^2,
\end{align*}
where $\ez_t = \eM h_t$ is an approximation of $S \sz_t$.
The ground truth for $\eN_2$ is $\sN_2 = [S \sA S^{\top}, S \sB]^{\top} [S \sA S^{\top}, S \sB]$, so its approximate factorization recovers $[S_2 \sA S^{\top}, S_2 \sB]$ for some orthogonal matrix $S_2$. In a similar way to \cosysid{}, we still need to fit an alignment matrix $S_3 = S S_2^{\top}$ to align the coordinates. Let $\tA$, $\tB$ denote the system parameters recovered from $\eN_2$. The linear regression we now need to solve is from $([\tA, \tB][\ez_{t}; u_{t}])_{t=H}^{T+H-1}$ to $(\ez_{t+1})_{t=H}^{T+H-1}$. However, without further assumptions, $[\sA, \sB]$ does not necessarily have full row rank, and hence, neither does $[\tA, \tB]$, in which case recovering the entire $S_3$ is impossible.

On the other hand, for \cosysid{} (Algorithm~\ref{alg:cosysid}), the ground truth of $\eM_1$ is $\sM_1 = [\sA \sM, \sB]$, which is guaranteed to have full row rank by the same argument as the proof of Proposition~\ref{prp:full-rank-cov}, since $\sM_1 [h_t; u_t]$ estimates $\sz_{t+1}$, which has full-rank covariance. Hence, recovering $S_0 = S S_1^{\top}$ \mbox{is feasible}. 

\section{Theoretical guarantees and proofs} 

The following Theorem~\ref{thm:main-poly} shows that both \corele{} and \coreli{} can provably solve unknown LQG control. 

\begin{theorem}  \label{thm:main-poly}
    Given an unknown LQG control problem defined by~\eqref{eqn:polti} and~\eqref{eqn:cost}, under Assumption~\ref{asp:asp}, 
    for a given $p \in (0, 1)$, if we run \corele{} (Algorithm~\ref{alg:lmbc} with~\eqref{eqn:sys-id}) or \coreli{} (Algorithm~\ref{alg:lmbc} with Algorithm~\ref{alg:cosysid}) for $T \ge \poly(d_x, d_y, d_u, \log(T / p))$, $H = \Omega(\frac{\log(\alpha H (d_y + d_u) T/p)}{\log(1/\rho)})$, 
    and $\sigma_u = \Theta(1)$, then with probability at least $1 - p$, the learned representation function $\eM$ is $\poly(H, d_x, d_u, d_y, \log(T/p))$-optimal, 
    and the overall output policy $\epi = (\eM, \eK)$ satisfies
    \begin{align*}
        J(\epi) - J(\spi) 
        = \bigo(\poly(H, d_x, d_u, d_y, \log(T/p)) T^{-1}).
    \end{align*}
\end{theorem}

\vspace*{2pt}
We defer the proof of Theorem~\ref{thm:main-poly} to \S\ref{sec:main-proof}.
Compared with the time-varying setting in Part I of this work, the bounds here do not have a \emph{separation} between the initial steps and future steps, where for the initial several steps, as the system has not been fully excited, the bounds were much worse. This is due to the fact that in the time-invariant setting, the representation function and the latent model are both \emph{stationary}.

On the other hand, to learn such stationary functions across different time steps, we need to aggregate \emph{correlated} data along a single trajectory, which poses new significant challenges for the analysis. A major effort to overcome such difficulties involves proving a new result on the persistency of excitation (Lemma~\ref{lem:per-exc}) using the small-ball method~\citep{mendelson2015learning, simchowitz2018learning}, which will be discussed further in \S\ref{sec:per-exc} with more details.

Compared with common system identification methods based on learning the \emph{Markov parameters}~\citep{oymak2019non, simchowitz2019learning}, the error bounds of the system parameters produced by \coreli{} (or \corele{}) have the same dependence on $T$, but worse dependence on system dimensions.
Moreover, to establish persistency of excitation, \coreli{} (or \corele{}) requires a larger burn-in period. 
These relative sample inefficiencies are the price we pay for cost-driven state representation learning, which is only supervised by \emph{scalar-valued} costs that are \emph{quadratic} in history, instead of \emph{vector-valued} observations that are \emph{linear} in history. Hence, we have to address the more challenging problem of {\it quadratic regression}, which lifts the dimension of the optimization problem. 
On the other hand, cost-driven state representation learning avoids learning the observation-reconstruction function $\sC$, and can learn task-relevant representations in more complex settings, as demonstrated by the empirical studies. 

\subsection{Proposition on multi-step cumulative costs}

The following proposition shows the relationship between $\oc_t$ and $h_t$, which is important for later analyses. 

\begin{proposition}  \label{prp:multi-step-cost}
    Given an LQG control problem satisfying Assumption~\ref{asp:asp}, let $\sM$ parameterize the optimal state representation function under the normalized parameterization.
    Recall that $\alpha := \max(\alpha(\sA), \alpha(\soA))$ and $\rho := \max(\rho(\sA), \rho(\soA))$.
    If we apply $u_t \sim \gauss(0, \sigma_u^2 I)$, then for any $t \ge H$, 
    \begin{align*}
        \oc_t := \nlsum_{\tau=t}^{t+d_x-1} (c_{\tau} - \|u_{\tau}\|_{\sR}^2) = \|\sM h_t\|^2 + \odelta_t + \sob + \overe_t,
    \end{align*}
    where $\odelta_t = \bigo(\alpha^2\rho^H \log(T/p))$ is a small error term, $\sob = \bigo(d_x)$ is a positive constant, and $\overe_t$ is a zero-mean subexponential random variable with $\|\overe_t\|_{\psi_1} = \bigo(d_x^{3/2})$. Moreover, let $\of_t = [\svec(h_t h_t^{\top}); 1]$. There exists a dimension-free constant $a > 0$ such that for a given $p \in (0, 1)$, as long as $H \ge \max(d_x - 1, \frac{a \log(\alpha T \log(T/p))}{\log(1/\rho)})$, $(\overe_t)_{t\ge H}$ satisfy
    \begin{align*}
        \Big\|\nlsum_{t=H}^{T+H-1} \of_t \overe_t\Big\| = \bigo(d_x^{3/2} d_h H^{1/2} T^{1/2}\log^{1/2}(H/p)).
    \end{align*}
\end{proposition}

\vspace*{2pt}
\begin{proof}
    Below we use $\Phi_{c, \ell}$ as a shorthand for $\Phi_{c, \ell}(\sA, \sB)$ and for $\ell = 0$, the term $\Phi_{c, \ell}u_{(t+\ell-1):t}$ is seen as zero.
    By definition, for $\ell \ge 0$, we have
    \begin{align*}
        & c_{t+\ell} - \|u_{t+\ell}\|_{\sR}^2 = \|x_{t+\ell}\|_{\sQ}^2 \\
        =\;& \Big\|(\sA)^{\ell} x_t + \Phi_{c, \ell}[u_{t+\ell-1}; \ldots; u_t] + \nlsum_{i=1}^{\ell} (\sA)^{i-1} w_{t+\ell-i}\Big\|_{\sQ}^2 \\
        \overset{(i)}{=}\;& \|(\sA)^{\ell} x_t\|_{\sQ}^2 + \|\Phi_{c, \ell} u_{(t+\ell-1):t}\|_{\sQ}^2 + \Big\|\nlsum_{i=1}^{\ell} (\sA)^{i-1} w_{t+\ell-i}\Big\|_{\sQ}^2,
    \end{align*}
    where $(i)$ is due to the independence of the three terms.

    Substituting $x_t = \sz_t + (x_t - \sz_t)$ and $\sz_t = \sM h_t + \delta_t$ into the above equation, we have 
    \begin{align*}
        c_{t+\ell} - \|u_{t+\ell}\|_{\sR}^2
        =\;& \|(\sA)^{\ell} \sz_t\|_{\sQ}^2 + \|(\sA)^{\ell} (x_t - \sz_t)\|_{\sQ}^2 + 2 \ipc{(\sA)^{\ell} \sz_t}{(\sA)^{\ell} (x_t - \sz_t)}_{\sQ} \\
        &\quad + \|\Phi_{c, \ell} u_{(t+\ell-1):t}\|_{\sQ}^2 + \Big\|\nlsum_{i=1}^{\ell} (\sA)^{i-1} w_{t+\ell-i}\Big\|_{\sQ}^2 \\
        =\;& \|(\sA)^{\ell} (\sM h_t + \delta_t)\|_{\sQ}^2 + \|(\sA)^{\ell} (x_t - \sz_t)\|_{\sQ}^2 + 2 \ipc{(\sA)^{\ell} \sz_t}{(\sA)^{\ell} (x_t - \sz_t)}_{\sQ} \\
        &\quad + \|\Phi_{c, \ell} u_{(t+\ell-1):t}\|_{\sQ}^2 + \Big\|\nlsum_{i=1}^{\ell} (\sA)^{i-1} w_{t+\ell-i}\Big\|_{\sQ}^2 \\
        =\;& \|(\sA)^{\ell} \sM h_t\|_{\sQ}^2 + \delta_{t, \ell} + \starb_{\ell} + e_{t, \ell},
    \end{align*}
    where $\delta_{t, \ell} := \| (\sA)^{\ell} \delta_t\|_{\sQ}^2 + 2 \ipc{(\sA)^{\ell} \sM h_t}{(\sA)^{\ell} \delta_t}_{\sQ}$ is a small term, and 
    \begin{align*}
        \starb_{\ell} :=\;& \E\Big[ \|(\sA)^{\ell} (x_t - \sz_t)\|_{\sQ}^2 + \|\Phi_{c, \ell} u_{(t+\ell-1):t}\|_{\sQ}^2 + \Big\|\nlsum_{i=1}^{\ell} (\sA)^{i-1} w_{t+\ell-i}\Big\|_{\sQ}^2 \Big], \\
        e_{t, \ell} :=\;& \|(\sA)^{\ell} (x_t - \sz_t)\|_{\sQ}^2 + 2\ipc{(\sA)^{\ell} \sz_t}{(\sA)^{\ell} (x_t - \sz_t)}_{\sQ} \\
        &\quad + \|\Phi_{c, \ell} u_{(t+\ell-1):t}\|_{\sQ}^2 + \Big\|\nlsum_{i=1}^{\ell} (\sA)^{i-1} w_{t+\ell-i}\Big\|_{\sQ}^2 - \starb_{\ell}.
    \end{align*}

    Note that $\starb_{\ell}$ is not a function of time step $t$ and $e_{t, \ell}$ is a zero-mean subexponential random variable with $\|e_{t, \ell}\|_{\psi_1} = \bigo(d_x^{1/2})$ by~\citep[Lemma 11]{tian2022cost}. 
    Define filtration $(\cF_t)_{t\geq 0}$ with
    \begin{align*}
        \cF_t := \sigma(x_0, y_0, u_0, x_1, y_1, \ldots, u_{t-1}, x_{t}, y_t).
    \end{align*}
    Then, $e_{t, \ell}$ is $\cF_{t+\ell}$-measurable. 
    Under the normalized parameterization, where $\nlsum_{\ell=0}^{d_x-1} ((\sA)^\ell)^{\top} \sQ (\sA)^\ell = I$, we have 
    \begin{align*}
        \oc_t = 
        \nlsum_{\tau = t}^{t+d_x-1} (c_{\tau} - \|u_{\tau}\|_{\sR}^2)
        = \|\sM h_t\|^2 + \odelta_{t} + \sob + \overe_{t},
    \end{align*}
    where we use 
    \begin{align*}
        \nlsum_{\ell = 0}^{d_x - 1} \ipc{(\sA)^{\ell} \sM h_t}{(\sA)^{\ell} \delta_t}_{\sQ} 
        = \nlsum_{\ell = 0}^{d_x - 1} (\sM_t h_t)^{\top} ((\sA)^{\ell})^{\top} \sQ (\sA)^{\ell} \delta_t 
        = (\sM_t h_t)^{\top} \delta_t
        = \ipc{\sM h_t}{\delta_t}
    \end{align*}
    and define
    \begin{align*}
        \odelta_t := \nlsum_{\ell=0}^{d_x-1} \delta_{t, \ell} = \|\delta_t\|^2 + 2\ipc{\sM h_t}{\delta_t},
    \end{align*}
    \vspace*{-1.4em}
    \begin{align*}
        \sob := \nlsum_{\ell=0}^{d_x-1} \starb_{\ell} = \E\Big[ \|x_t - \sz_t\|^2
        + \nlsum_{\ell=0}^{d_x-1} \Big(\|\Phi_{c, \ell} u_{(t+\ell-1):t}\|_{\sQ}^2
        + \Big\|\nlsum_{i=1}^{\ell} (\sA)^{i-1} w_{t+\ell-i}\Big\|_{\sQ}^2\Big) \Big],
    \end{align*}
    \vspace*{-1.4em}
    \begin{align*}
        \overe_{t} := \nlsum_{\ell=0}^{d_x-1} e_{t, \ell} =\;& \|x_t - \sz_t\|^2 + 2\ipc{\sz_t}{x_t - \sz_t} \\
        &\quad + \nlsum_{\ell=0}^{d_x-1} \Big(\|\Phi_{c, \ell} u_{(t+\ell-1):t}\|_{\sQ}^2 + \Big\|\nlsum_{i=1}^{\ell} (\sA)^{i-1} w_{t+\ell-i}\Big\|_{\sQ}^2 \Big) - \sob.
    \end{align*}

    Since $\delta_t = (\soA)^H \sz_{t-H}$,  we have 
    \begin{align*}
        \|\odelta_t\| = \bigo(\alpha^2 \rho^{2H} \log(T/p) + \alpha \rho^H \log(T/p))
        = \bigo(\alpha^2 \rho^{H} \log(T/p) ).
    \end{align*}
    Moreover, constant $\sob = \bigo(d_x)$, $\overe_t$ is a zero-mean subexponential random variable with $\|\overe_t\|_{\psi_1} = \bigo(d_x^{3/2})$, and the random process $(\overe_t)_{t\ge H}$ is adapted to the filtration $(\cF_{t+d_x-1})_{t\ge H}$.

    However, the concentration of $\sum_{t=H}^{T+H-1} \of_t \overe_t$ or even $\sum_{t=H}^{T+H-1} \overe_t$ is highly nontrivial, in that $(\overe_t)_{t\ge H}$ is not a martingale difference sequence.
    Below we develop the idea that random variables that are widely separated in a mixing stochastic process are nearly independent, in order to show the concentration of $\sum_{t=H}^{T+H-1} \of_t \overe_t$.
    Specifically, we divide the time steps into $\oH = 2(H+d_x-1) = \bigo(H)$ partitions.  For the $i$th partition with $H\le i < H+\oH$, the indices are given by $(i+j \oH)_{j\ge 0}$.  
    To obtain independent random variables $(g_{i+j\oH})_{j\ge 0}$, we apply the Gram-Schmidt process to $(\of_{i+j\oH} \overe_{i+j\oH})_{j \ge 0}$, which is adapted to $(\cF_{i+j\oH+d_x-1})_{j\ge 0}$. That is, $g_{i} = \of_{i} \overe_{i}$, and for $j \ge 0$, 
    \begin{align*}
        g_{i+(j+1)\oH} = \of_{i+(j+1)\oH} \overe_{i+(j+1)\oH} - \E[\of_{i+(j+1)\oH} \overe_{i+(j+1)\oH} \given \cF_{i+j\oH+d_x-1}].
    \end{align*}

    Thus, we further have  
    \begin{align}  \label{eqn:sum-fe}
        \nlsum_{t=H}^{T+H-1} \of_t \overe_t
        = \nlsum_{t=H}^{T+H-1} g_t + \nlsum_{t=H}^{T+H-1} \E[\of_{t+\oH} \overe_{t+\oH} \given \cF_{t+d_x-1}].
    \end{align}
    By~\citep[Proposition 2]{zhang2022sharper}, since each dimension of $\of_t$ is subexponential with mean and the subexponential norm both bounded by $\bigo(1)$, each dimension of $g_{i+j\oH}$ is $\frac{1}{2}$-sub-Weibull, with the sub-Weibull norm being $\bigo(d_x^{3/2})$.
    By applying the sub-Weibull concentration bound~\citep[Theorem 3.1]{hao2019bootstrapping} to each of the $1 + d_h(d_h + 1)/2$ dimensions in $(g_{i+j\oH})_{j \ge 0}$, we have that with probability at least $1-p$,  
    \begin{align*}
        \Big\|\nlsum_{j\ge 0} g_{i+j\oH}\Big\| = \bigo(d_x^{3/2} d_h (T/\oH)^{1/2} \log^{1/2}(1/p)). 
    \end{align*}
    Repeating the argument for each $H\le i < H+\oH$ and by the union bound, we have  
    \begin{equation}  \label{eqn:bd-sum-g}
        \begin{aligned}
            \Big\|\nlsum_{t=H}^{T+H-1} g_t\Big\| = \bigo(d_x^{3/2} d_h \oH(T/\oH)^{1/2} \log^{1/2}(\oH/p))
            = \bigo(d_x^{3/2} d_h H^{1/2} T^{1/2} \log^{1/2}(H/p) ).
        \end{aligned}
    \end{equation}

    It remains to bound the residuals of the Gram-Schmidt process. To this end, we first split $\of_{t+\oH}$ and $\overe_{t+\oH}$ into two parts, with one being measurable with respect to,  and the other independent of,  $\cF_{t+d_x-1}$. 
    By definition, for $k \ge 1$, 
    \begin{align*}
        y_{t+k} = \sC ((\sA)^k x_t + \Phi_{c, k} u_{(t+k-1):t}
        + \nlsum_{i=1}^{k} (\sA)^{i-1} w_{t+k-i}) + v_{t+k}
        = \sC (\sA)^k x_t + \xi_{t, k}^y,
    \end{align*}
    where $\xi_{t, k}^y$ is independent of $\cF_t$, and it is a zero-mean Gaussian random vector with the operator norm of the covariance matrix being $\bigo(1)$,  due to the stability of $\sA$.

    Recall that $f_t = \svec(h_t h_t^{\top} )$ and $h_t = [u_{(t-H):(t-1)}; y_{(t-H+1):t}]$. 
    Let $h_{t+\oH} = s_{t+\oH} + \xi_{t+\oH}^h$, where 
    \begin{align*}
        s_{t+\oH} =\;& [\sC (\sA)^{\oH-H-d_x+2} x_{t+d_x-1}; \ldots; \sC (\sA)^{\oH-d_x+1} x_{t+d_x-1}; 0_{H d_u}], \\
        \xi_{t+\oH}^h =\;& [\xi_{t+d_x-1, \oH - H - d_x + 2}^y; \ldots; \xi_{t+d_x-1, \oH-d_x+1}^y; u_{(t+\oH-H):(t+\oH-1)}], 
    \end{align*}
    and $\xi_{t+\oH}^h$ is independent of $\cF_{t+d_x-1}$, and it is a zero-mean Gaussian random vector with the variance of each dimension bounded by $\bigo(1)$.
    Then, we have 
    \begin{equation}  \label{eqn:of}
        \begin{aligned}
            \of_{t+\oH}
            =\;& [\svec(h_{t+\oH} h_{t+\oH}^{\top}); 1] \\
            =\;& [\svec(s_{t+\oH} s_{t+\oH}^{\top} + s_{t+\oH} (\xi_{t+\oH}^h)^{\top} + \xi_{t+\oH}^h s_{t+\oH}^{\top} + \xi_{t+\oH}^h (\xi_{t+\oH}^h)^{\top}); 1].
        \end{aligned}
    \end{equation}

    We now turn to analyze $\overe_t$. Notice that 
    \begin{align*}
        x_{t+1} - \sz_{t+1}
        =\;& \sA x_t + \sB u_t + w_t - (\sA \sz_t + \sB u_t + \sL (y_{t+1} - \sC (\sA \sz_t + \sB u_t))) \\
        =\;& \sA (x_t - \sz_t) + w_t - \sL (\sC (\sA x_t + \sB u_t + w_t) + v_{t+1} - \sC (\sA \sz_t + \sB u_t )) \\
        =\;& \soA (x_t - \sz_t) + (I - \sL \sC) w_t + v_{t+1},
    \end{align*}
    we thus have 
    \begin{align*}
        x_{t+\oH} - \sz_{t+\oH}
        = (\soA)^{\oH-d_x+1} (x_{t+d_x-1} - \sz_{t+d_x-1}) + \delta_{t+\oH}^x,
    \end{align*}
    where $\delta_{t+\oH}^x$ is independent of $\cF_{t+d_x-1}$, and it is  a zero-mean Gaussian random vector with the operator norm of the covariance matrix bounded by $\bigo(1)$,  due to the stability of $\soA$.
    Since
    \begin{align*}
        \sz_{t+1}
        =\;& \soA \sz_t + \soB u_t + \sL y_{t+1} \\
        =\;& \soA \sz_t + \soB u_t + \sL (\sC (\sA x_t + \sB u_t + w_t) + v_{t+1}) \\ 
        =\;& \sA \sz_t + \sL \sC \sA (x_t - \sz_t) + \sB u_{t} + \sL \sC w_t + \sL v_{t+1},
    \end{align*}
    we have 
    \begin{align*}
        \sz_{t+\oH} 
        = (\sA)^{\oH-d_x+1} \sz_{t+d_x-1} + \Phi_A (x_{t+d_x-1} - \sz_{t+d_x-1}) + \xi_{t+\oH}^{z},
    \end{align*}
    where $\Phi_A := \nlsum_{i=1}^{\oH-d_x+1} (\sA)^{\oH-d_x+1-i} \sL \sC \sA (\soA)^{i-1}$ and $\xi_{t+\oH}^z$ is independent of $\cF_{t+d_x-1}$, and a zero-mean Gaussian random vector with the operator norm of the covariance matrix bounded by $\bigo(1)$ due to the stability of $\sA$ and $\soA$.

    Hence, $\overe_{t+\oH}$ can be expressed as 
    \begin{equation}  \label{eqn:oe}
        \begin{aligned}
            \overe_{t+\oH} 
            =\;& \|(\soA)^{\oH-d_x+1} (x_{t+d_x-1} - \sz_{t+d_x-1}) + \xi_{t+\oH}^{x} \|^2 \\
            &\quad + 2 \ipc{(\sA)^{\oH-d_x+1} \sz_{t+d_x-1} + \Phi_A (x_{t+d_x-1} - \sz_{t+d_x-1}) \\
            &\quad + \xi_{t+\oH}^{z}}{(\soA)^{\oH-d_x+1} (x_{t+d_x-1} - \sz_{t+d_x-1}) + \xi_{t+\oH}^{x}} \\
            &\quad + \xi_{t+\oH}^e - \E\Big[\|(\soA)^{\oH-d_x+1} (x_{t+d_x-1} - \sz_{t+d_x-1}) + \xi_{t+\oH}^{x} \|^2\Big],
        \end{aligned}
    \end{equation}
    where 
    \begin{align*}
        \xi_{t+\oH}^e
        :=\;& \nlsum_{\ell=0}^{d_x-1} \Big(\|\Phi_{c, \ell} u_{(t+\oH+\ell-1):(t+\oH)}\|_{\sQ}^2 + \Big\|\nlsum_{i=1}^{\ell} (\sA)^{i-1} w_{t+\oH+\ell-i}\Big\|_{\sQ}^2 \Big)\\
        &\quad - \E\Big[ \nlsum_{\ell=0}^{d_x-1} \Big(\|\Phi_{c, \ell} u_{(t+\oH+\ell-1):(t+\oH)}\|_{\sQ}^2 + \Big\|\nlsum_{i=1}^{\ell} (\sA)^{i-1} w_{t+\oH+\ell-i}\Big\|_{\sQ}^2 \Big) \Big]
    \end{align*}
    is independent of $\cF_{t+d_x-1}$, and it is a zero-mean subexponential random variable with $\bigo(d_x^{3/2})$ subexponential norm due to the stability of $\sA$.

    Notice that
    \begin{align*}
        \E[\of_{t} \overe_t] = \E[\E[\of_{t} \overe_t \given y_{0:t}, u_{0:(t-1)}]]
        = \E[\of_t \E[\overe_t \given y_{0:t}, u_{0:(t-1)}]] = 0.
    \end{align*}
    Then, by substituting~\eqref{eqn:of} and~\eqref{eqn:oe}, we have
    \begin{align*}
        &\E[\of_{t+\oH} \overe_{t+\oH} \given \cF_{t+d_x-1}] \\
        =\;& \E[\of_{t+\oH} \overe_{t+\oH} \given \cF_{t+d_x-1}] - \E[\of_{t+\oH} \overe_{t+\oH}] \\
        =\;& \E\Big[\Big([\svec(s_{t+\oH} s_{t+\oH} + s_{t+\oH} (\xi_{t+\oH}^h)^{\top} + \xi_{t+\oH}^h s_{t+\oH}^{\top} + \xi_{t+\oH}^h (\xi_{t+\oH}^h)^{\top}); 1]\Big) \\
        &\quad \cdot \Big(\|(\soA)^{\oH-d_x+1} (x_{t+d_x-1} - \sz_{t+d_x-1}) + \xi_{t+\oH}^{x}\|^2 + 2 \ipc{(\sA)^{\oH-d_x+1} \sz_{t+d_x-1} \\
        &\quad + \Phi_A (x_{t+d_x-1} - \sz_{t+d_x-1}) + \xi_{t+\oH}^{z}}{(\soA)^{\oH-d_x+1} \cdot (x_{t+d_x-1} - \sz_{t+d_x-1}) + \xi_{t+\oH}^{x}} \\
        &\quad + \xi_{t+\oH}^e \Big) \;\big|\; \cF_{t+d_x-1} \Big] - \E\Big[\Big([\svec(s_{t+\oH} s_{t+\oH} + s_{t+\oH} (\xi_{t+\oH}^h)^{\top} + \xi_{t+\oH}^h s_{t+\oH}^{\top} + \xi_{t+\oH}^h (\xi_{t+\oH}^h)^{\top}); 1]\Big) \\
        &\quad \cdot \Big(\|(\soA)^{\oH-d_x+1} (x_{t+d_x-1} - \sz_{t+d_x-1}) + \xi_{t+\oH}^{x}\|^2 + 2 \ipc{(\sA)^{\oH-d_x+1} \sz_{t+d_x-1} \\
        &\quad + \Phi_A (x_{t+d_x-1} - \sz_{t+d_x-1}) + \xi_{t+\oH}^{z}}{(\soA)^{\oH-d_x+1} \cdot (x_{t+d_x-1} - \sz_{t+d_x-1}) + \xi_{t+\oH}^{x}} + \xi_{t+\oH}^e \Big) \Big],
    \end{align*}
    where the terms completely independent of $\cF_{t+d_x-1}$ cancel each other, and all other terms contain at least one of $(\sA)^{\oH-H-d_x+2}$, $(\soA)^{\oH-d_x+1}$ and $\Phi_A$, with each of the $1+d_h(d_h + 1)/2$ dimensions being the product of two subexponential random variables. 
    Hence, with probability at least $1-p$,
    \begin{equation}
        \begin{aligned}
        \|\E[\of_{t+\oH} \overe_{t+\oH} \given \cF_{t+d_x-1}]\|
        = \bigo(\alpha d_x^{3/2} d_h \rho^{\oH - H - d_x + 2 } \log^3(T/p))
        = \bigo(\alpha d_x^{3/2} d_h \rho^{H} \log^2(T/p)).\label{eqn:bd-sum-res}
        \end{aligned}
    \end{equation}

    Finally, combining~\eqref{eqn:sum-fe} with the bounds in~\eqref{eqn:bd-sum-g} and~\eqref{eqn:bd-sum-res}, we have
    \begin{align*}
        \|\nlsum_{t=H}^{T+H-1} \of_t \overe_t\| = \bigo(d_x^{3/2} d_h (H^{1/2} T^{1/2}\log^{1/2}(H/p) + \alpha \rho^H T \log^2(T/p)) ).
    \end{align*}
    Hence, as long as $H \ge \frac{a\log(\alpha T \log(T/p))}{\log(1/\rho)}$ for some dimension-free constant $a > 0$, we have
    \begin{align*}
        \|\nlsum_{t=H}^{T+H-1} \of_t \overe_t\| = \bigo(d_x^{3/2} d_h H^{1/2} T^{1/2} \log^{1/2}(H/p)),
    \end{align*}
    completing the proof. 
\end{proof}

\vspace*{-6pt}
\subsection{Persistency of excitation}
\label{sec:per-exc}

Central to the analysis of \corele{} and \coreli{} is the finite-sample characterization of the \emph{quadratic regression} problem \eqref{eqn:alg-qr}. To this end, 
notice that 
\begin{align*}
    \|h_t\|_N^2 = \ipc{N}{h_t h_t^{\top}}_F  = \ipc{\svec(N)}{\svec(h_t h_t^{\top})},
\end{align*}
which means that this quadratic regression is essentially a linear regression problem in terms of $[\svec(N); b_0]$.
A major difficulty in the analysis is to establish persistency of excitation for $([\svec(h_t h_t^{\top}); 1])_{t\ge H}$, 
meaning that the minimum eigenvalue of the Gram matrix 
\begin{align*}
    \nlsum_{t=H}^{T+H-1} [\svec(h_t h_t^{\top}); 1] [\svec(h_t h_t^{\top})^{\top}, 1]
\end{align*}
grows linearly in the data size $T$. This is needed to ensure the uniqueness and convergence of the  parameter estimation.

A linear lower bound on $\lambda_{\min}(\sum_{t=H}^{T+H-1} h_t h_t^{\top})$ is a known result for the identification of partially observable linear dynamical systems, see the recent overview in \citep{tsiamis2022statistical}. 
In our case, however, elements of $\svec(h_t h_t^{\top})$ are \emph{products} of Gaussians, making the analysis difficult.
If $(h_t)_{t\ge H}$ are independent, which is the case if they are from multiple independent trajectories, the result has been established in~\citep{jadbabaie2021time} and Part I of this work. It can also be proved with the matrix Azuma inequality~\citep{tropp2012user}. 
Here, by contrast, we need to aggregate \emph{correlated}  data to estimate a set of \emph{stationary} parameters.
In sum, the difficulty we face results from both products of Gaussians and the \mbox{data dependence}.

In principle, given enough burn-in time, the state $x_t$, and hence the observation $y_t$ and the truncated history $h_t$, converge to the steady-state distributions, and samples with an interval of the order of mixing time are approximately independent~\citep{levin2017markov}; our proof of Propositions~\ref{prp:multi-step-cost} has been built upon this idea. Hence, a linear lower bound is expected. However, the bound yielded by such an analysis deteriorates as the system becomes less stable and the mixing time increases.
To eschew such dependence, the so-called \emph{small-ball} method is introduced in~\citep{simchowitz2018learning}. We take the same route, while establishing different arguments to handle the products of Gaussian \mbox{random variables}.

Let us first recall the block martingale small-ball condition~\citep[\mbox{Definition 2.1}]{simchowitz2018learning}.
\begin{definition}[Block martingale small-ball (BMSB) condition{}]
    Let $(f_t)_{t\ge 1}$ be a stochastic process in $\R^d$ adapted to the filtration $(\cF_{t})_{t\ge 1}$. We say $(f_t)_{t\ge 1}$ satisfies the $(k, \Gamma, q)$-BMSB condition for $k \in \N^{+}$, $\Gamma \succ 0$ and $q > 0$, if for any $t \ge 1$, for any fixed unit vector $v \in \R^d$, 
    $\frac{1}{k} \sum_{i=1}^{k} \p(|\ip{f_{t+i}}{v}| \ge \|v\|_{\Gamma} \given \cF_t) \ge q$ almost surely.
\end{definition}

The key lemma below shows that $([\svec(h_t h_t^{\top}); 1])_{t\ge H}$ is persistently exciting using the BMSB condition.

\begin{lemma}  \label{lem:per-exc}
    Let $h_t = [y_{(t-H+1):t}; u_{(t-H):(t-1)}]$ be the $H$-step history at time step $t \ge H$ in system~\eqref{eqn:polti} with $u_t \sim \gauss(0, \sigma_u^2 I)$ for $t \ge 0$. 
    Define filtration $(\cF_t)_{t\geq 0}$ with
    $\cF_t := \sigma(x_0, y_0, u_0, x_1, y_1, \ldots, u_{t-1}, x_t, y_t)$.
    Define $f_t := \svec(h_{t} h_t^{\top})$ and $\of_t := [f_t; 1]$, adapted to $(\cF_t)_{t\ge H}$.
    Recall that for square matrix $A$, $\alpha(A) := \sup_{k \ge 0} \|(A)^k\|_2 \rho(A)^{-k}$.
    There exist dimension-free constants $a_1, a_2 > 0$, such that for a given $p \in (0, 1)$, as long as $H \ge \frac{a_1 \log(d_h \alpha(\sA) \log(T/p))}{\log(\rho(\sA)^{-1})}$ and $T \ge a_2 H d_h^8 \log (d_h / p)$, we have 
    \begin{align*}
        \lambda_{\min}\Big(\nlsum_{t=H}^{T+H-1} \of_t \of_t^{\top}\Big) =  \Omega(d_h^{-9} T).
    \end{align*}
\end{lemma}

\vspace*{2pt}
\begin{proof}
    Since $\svec$ is a bijection, every vector $w \in \R^{d_h(d_h + 1)/2}$ corresponds to a symmetric matrix $D\in \R^{d_h \times d_h}$ with Frobenius norm $\|w\|$. Then, for any unit vector $v = [w; s]$ with $w \in \R^{d_h(d_h + 1)/2}$ and $s \in \R$,
    \begin{align*}
        \ip{\of_{t+i}}{v} = \ip{f_{t+i}}{w} + s
        = \ipc{\svec(h_{t+i} h_{t+i}^{\top})}{\svec{(D)}} + s
        = h_{t+i}^{\top} D h_{t+i} + s.
    \end{align*}
    Take $\Gamma = \gamma^2 I$ for some $\gamma > 0$ to be specified later. Then, $\|v\|_{\Gamma} = \gamma$. 
    It suffices to show that for $i > \oH$ for some $\oH > 0$, 
    \begin{align*}
        \p(|h_{t+i}^{\top} D h_{t+i} + s| \ge \gamma \given \cF_t) \ge q,
    \end{align*} 
    since if so, we have 
    \begin{align*}
        \frac{1}{2\oH} \nlsum_{i=1}^{2\oH} \p(|h_{t+i}^{\top} D h_{t+i} + s|
        \ge \gamma \given \cF_t)
        \ge \frac{1}{2\oH} \nlsum_{i=\oH+1}^{2\oH} \p(|h_{t+i}^{\top} D h_{t+i} + s|
        \ge \gamma \given \cF_t) \ge q/2,
    \end{align*}
    which means $(\of_t)_{t\ge H}$ is $(2\oH, \gamma^2 I, q/2)$-BMSB.

    Now let us take a close look at 
    \begin{align*}
        h_{t+i} = [y_{(t+i-H+1):(t+i)}; u_{(t+i-H):(t+i-1)}].
    \end{align*}
    Since 
    \begin{align*}
        y_{t+i} = \sC (\sA)^{i} x_t + \nlsum_{j=1}^{i} \sC (\sA)^{j} (\sB u_{t+i-j} + w_{t+i-j}) + v_{t+i},
    \end{align*}
    $y_{t+i} \given \cF_t$ is Gaussian with mean $\sC (\sA)^{i} x_t$ and covariance determined by $\sum_{j=1}^{i} \sC (\sA)^{j} (\sB u_{t+i-j} + w_{t+i-j}) + v_{t+i}$, where we note that $v_{t+i}$ is independent of all other random variables and has full-rank covariance.
    Hence, for $i \ge H$, $h_{t+i} \given \cF_t$ is Gaussian and has full-rank covariance. 
    Then intuitively, since $\|D\|_F = 1$, $|h_{t+i}^{\top} D h_{t+i}| \given \cF_t$ is a well-behaved random variable that can exceed some $\gamma > 0$ with a positive probability $q$. 

    Formally, let $\mu_{t, i} := \E[h_{t+i} \given \cF_t]$. By Lemma~\ref{lem:exp-abs-lb}, for $i \ge H$, 
    there exists some absolute constant $a > 0$, such that 
    \begin{align*}
        \E[|(h_{t+i} - \mu_{t, i})^{\top} D (h_{t+i} - \mu_{t, i}) + s| \given \cF_t]
        \ge a \min\{\sigma_u, \sigma_v\} d_h^{-3/2}.
    \end{align*}
    By the triangle inequality, we have
    \begin{align*}
        |(h_{t+i} - \mu_{t, i})^{\top} D (h_{t+i} - \mu_{t, i}) + s|
        =\;& |h_{t+i}^{\top} D h_{t+i} + \mu_{t, i}^{\top} D \mu_{t, i} - 2 h_{t+i}^{\top} D \mu_{t, i} + s| \\
        \le\;& |h_{t+i}^{\top} D h_{t+i} + s| + |\mu_{t, i}^{\top} D \mu_{t, i}| + 2 |h_{t+i}^{\top} D \mu_{t, i}|.
    \end{align*}
    Hence, 
    \begin{align*}
        \E[|h_{t+i}^{\top} D h_{t+i} + s| \given \cF_t]
        \ge
        a \min\{\sigma_u, \sigma_v\} d_h^{-3/2} - \E[|\mu_{t, i}^{\top} D \mu_{t, i}| + 2 |h_{t+i}^{\top} D \mu_{t, i}| \given \cF_t].
    \end{align*}

    Now we argue that for large enough $i$, $\E[|\mu_{t, i}^{\top} D \mu_{t, i}| + 2 |h_{t+i}^{\top} D \mu_{t, i}|]$ is negligible.
    Since matrix $\sA$ is stable, $\|\Cov(x_t)\|_2 = \bigo(1)$ for all $t \ge 0$. By the tail bound of sub-Gaussian random variable $\|x_t\|$ and the union bound, with probability at least $1 - p$, $\|x_t\| = \bigo(d_x^{1/2} \log(T/p))$ for all $0\le t \le T + H$.
    Hence, 
    \begin{align*}
        \|\sC (\sA)^i x_t\| = \bigo(\alpha(\sA)\rho(\sA)^i d_x^{1/2} \log(T/p)),
    \end{align*}
    where we recall that $\alpha(\sA) := \sup_{k \ge 0} \|(\sA)^k\|_2 \rho(\sA)^{-k}$ and $\|\sC\|_2$, $\|\sA\|_2$ are hidden in $\bigo(\cdot)$.
    Then, for $i \ge H$,
    \begin{align*}
        \E[|\mu_{t, i}^{\top} D \mu_{t, i}| + 2 |h_{t+i}^{\top} D \mu_{t, i}| \given \cF_t]
        =\;& |\ipc{\mu_{t, i} \mu_{t, i}^{\top}}{D}_F| + 2 \E[|\ipc{\mu_{t, i} h_{t+i}^{\top}}{D}_F| \given \cF_t] \\
        \le\;& \|\mu_{t, i} \mu_{t, i}^{\top}\|_F \cdot \|D\|_F + 2 \E[ \|\mu_{t, i} h_{t+i}^{\top}\|_F \cdot \|D\|_F \given \cF_t] \\
        =\;& \|\mu_{t, i}\|^2 + 2 \|\mu_{t, i}\| \cdot \E[\|h_{t+i}\| \given \cF_t].
    \end{align*}

    By definition, $\mu_{t, i}$ is the concatenation of $(\sC (\sA)^j x_t)_{i-H+1 \le j \le i}$ and zero vectors. Hence, \mbox{we have} 
    \begin{align*}
        \|\mu_{t, i}\| = \bigo(d_h^{1/2} \alpha(\sA)\rho(\sA)^i \log(T/p)).
    \end{align*}
    Choosing $H \ge \frac{a_1 \log(d_h \alpha(\sA) \log(T/p))}{\log(\rho(\sA)^{-1})}$ for some dimension-free constant $a_1 > 0$, such that for $i > 2 H$, we have
    \begin{align*}
        \|\mu_{t, i}\|^2 + 2 \|\mu_{t, i}\| \cdot \E[\|h_{t+i}\| \given \cF_t] \le a \min\{\sigma_u, \sigma_v\} d_h^{-3/2} / 2.
    \end{align*}
    Then, we obtain the desired lower bound that 
    \begin{align*}
        \E[|h_{t+i}^{\top} D h_{t+i} + s| \given \cF_t] \ge 
        a \min\{\sigma_u, \sigma_v\} d_h^{-3/2} / 2.
    \end{align*}

    On the other hand, since 
    \begin{align*}
        |h_{t+i}^{\top} D h_{t+i} + s| = \big|\ipc{D}{h_{t+i} h_{t+i}^{\top}}_F + s \big|
        \le \|D\|_F \|h_{t+i} h_{t+i}^{\top}\|_F + |s| \le h_{t+i}^{\top} h_{t+i} + |s|, 
    \end{align*}
    we have $\E[|h_{t+i}^{\top} D h_{t+i} + s|^2 \given \cF_t] \le 2 \E[\|h_{t+i}\|^4 \given \cF_t] + 2 s^2$. 
    Since $\|h_{t+i}\| \given \cF_t$ is sub-Gaussian with 
    \begin{align*}
        \|\|h_{t+i}\| \given \cF_t\|_{\psi_2} = \bigo(\|\E[h_{t+i} h_{t+i}^{\top} \given \cF_t]\|_2^{1/2}) = \bigo(1), 
    \end{align*}
    it follows that
    $\E[|h_{t+i}^{\top} D h_{t+i} + s|^2 \given \cF_t] = \bigo(1)$.
    By the Paley-Zygmund inequality, for $\beta \in [0, 1]$ \mbox{we have}
    \begin{align*}
        \p(|h_{t+i}^{\top} D h_{t+i} + s| \ge \beta a \min\{\sigma_u, \sigma_v\} d_h^{-3/2} / 2 \given \cF_t)
        = \Omega((1 - \beta)^2 a^2 d_h^{-3}),
    \end{align*}
    where the dependence on $\sigma_u$, $\sigma_v$ is hidden in $\Omega(\cdot)$.
    By taking $\beta = 1/2$, we can see that $(f_t)_{t\ge H}$ satisfies the $(k, \gamma^2 I, q)$-BMSB condition for $k = 4 H$, $\gamma = \Theta(d_h^{-3/2})$ and $q = \Theta(d_h^{-3})$.

    Following the analysis in~\citep[Appendix D]{simchowitz2018learning}, by lower bounding 
    \begin{align*}
        \inf_{v: \|v\| = 1} \nlsum_{t=H}^{T+H-1} \ip{v}{\of_t}^2
    \end{align*}
    using a covering argument~\citep[Lemma 4.1]{simchowitz2018learning}, we can show that for a given $p\in (0, 1)$, as long as $T \ge a_2 H d_h^8 \log(d_h / p) = \Omega(k d_h^2 q^{-2} \log (d_h / (\gamma q p)))$ for some dimension-free constant $a_2 > 0$, we have with probability at least $1 - p$ that 
    \begin{align*}
        \lambda_{\min}\Big(\nlsum_{t=H}^{T+H-1} \of_t \of_t^{\top}\Big) = \Omega(\gamma^2 q^2 T ) = \Omega(d_h^{-9} T),
    \end{align*}
    which completes the proof. 
\end{proof}

Crucial for the  proof above is Lemma~\ref{lem:exp-abs-lb}, a lower bound on the expectation of Gaussian quadratic forms, which might be of independent interest.

\noindent\textbf{Lower bound for Gaussian quadratic forms.}

\begin{lemma}  \label{lem:sq-g-lb}
    Let $z_1, z_2, \ldots, z_{d}$ be independent standard Gaussian random variables. Let $v = [v_1, v_2, \ldots, v_{d+1}]^{\top} \in \s^{d}$  be a $(d+1)$-dimensional unit vector. Then, 
    \begin{align*}
        \inf\nolimits_{v \in \s^{d}}\E\Big[\Big|v_{d+1} + \nlsum_{i=1}^{d} v_i z_i^2\Big|\Big] \ge 0.8 \cdot  d^{-3/2}.
    \end{align*}
\end{lemma}
\begin{proof}
    Let us consider the value of $v_{d+1}$. 
    Since $\E[z_i^2] = 1$ for all $1\le i \le d$, we have
    \begin{align*}
        \E\Big[\Big|\nlsum_{i=1}^{d} v_i z_i^2\Big|\Big]
        \le \nlsum_{i=1}^{d} |v_i| 
        \le \sqrt{d \nlsum_{i=1}^{d} v_i^2}
        \le \sqrt{d (1 - v_{d+1}^2)}. 
    \end{align*}
    Then, we have 
    \begin{align*}
        \E\Big[\Big|v_{d+1} + \nlsum_{i=1}^{d} v_i z_i^2\Big|\Big]
        \ge |v_{d+1}| - \E\Big[\Big|\nlsum_{i=1}^{d} v_i z_i^2\Big|\Big]
        \ge |v_{d+1}| - \sqrt{d(1 - v_{d+1}^2)}.
    \end{align*}
    Hence, if $|v_{d+1}| \ge 2\sqrt{d/(4d + 1)}$, we have $\sqrt{d(1 - v_{d+1}^2)} \le |v_{d+1}|/2$. It follows that 
    \begin{align*}
        \E\Big[\Big|v_{d+1} + \nlsum_{i=1}^{d} v_i z_i^2\Big|\Big]
        \ge \frac{|v_{d+1}|}{2} \ge \sqrt{\frac{d}{4d + 1}} \ge \frac{1}{\sqrt{5}}.
    \end{align*}

    Below we consider the case where $|v_{d+1}| < 2\sqrt{d/(4d + 1)}$.
    Let $\sign(\cdot)$ denote the sign function. Let $\cI^{+} := \{i: \sign(v_i) = 1,~ 1\le i\le d \}$ and $\cI^{-} := \{i: \sign(v_i) = -1,~ 1\le i\le d \}$ be the index sets of positive and negative values among $(v_i)_{i=1}^{d}$. Then,
    \begin{align*}
        \E\Big[\Big|v_{d+1} + \nlsum_{i=1}^{d} v_i z_i^2 \Big|\Big]
        =\;& \E\Big[\Big|v_{d+1} + \nlsum_{i=1}^{d} |v_i| \sign(v_i) z_i^2 \Big|\Big] \\
        =\;& \E\Big[\Big|v_{d+1} + \nlsum_{i\in \cI^{+}} |v_i| z_i^2 - \nlsum_{j \in \cI^{-}} |v_j| z_j^2 \Big|\Big].
    \end{align*}
    For a given $v$, since $(z_i^2)_{i=1}^{d}$ have identical distributions, $\E\Big[\Big|v_{d+1} + \nlsum_{i\in \cI^{+}} |v_i| z_i^2 - \nlsum_{j \in \cI^{-}} |v_j| z_j^2 \Big|\Big]$ has the same value under permutations of $(v_i)_{i\in \cI^{+}}$ and $(v_j)_{j\in \cI^{-}}$. 
    Summing over all the permutations of $(v_i)_{i\in \cI^{+}}$ and $(v_j)_{j\in \cI^{-}}$ gives 
    \begin{align*}
        d \E\Big[\Big|v_{d+1} + \nlsum_{i=1}^{d} v_i z_i^2 \Big|\Big]
        \ge \E\Big[\Big|d \cdot v_{d+1} + \big(\nlsum_{i\in \cI^{+}} |v_i|\big) \nlsum_{i \in \cI^{+}} z_i^2
        - \big(\nlsum_{j \in \cI^{-}} |v_j|\big) \nlsum_{j \in \cI^{-}} z_j^2 \Big|\Big].
    \end{align*}

    Hence, we further have 
    \begin{align*}
        \E\Big[\Big|v_{d+1} + \nlsum_{i=1}^{d} v_i z_i^2 \Big|\Big]
        \ge \frac{1}{d} \Big(\nlsum_{i=1}^{d} |v_i| \Big) \E\Big[ \Big|d \cdot v_{d+1} + \nlsum_{i=1}^{d} \sign(v_i) z_i^2 \Big|\Big]
    \end{align*}
    Since $\sum_{i=1}^{d} |v_i| \ge (\sum_{i=1}^{d} v_i^2)^{1/2} = (1 - v_{d+1}^2)^{1/2}$, we have 
    \begin{align*}
        \E\big[\big|v_{d+1} + \nlsum_{i=1}^{d} v_i z_i^2 \big|\big] 
        \ge \frac{(1-v_{d+1}^2)^{1/2}}{d} \inf\nolimits_{w\in \{\pm 1\}^{d}} \E\big[\big| d \cdot v_{d+1} + \nlsum_{i=1}^{d} w_i z_i^2 \big|\big].
    \end{align*}
    It remains to lower bound $\inf_{w\in \{\pm 1\}^{d}} \E\big[\big| d \cdot v_{d+1} + \nlsum_{i=1}^{d} w_i z_i^2 \big|\big]$.
    By symmetry, for any pair $w_i \neq w_j$, the expectation remains the same if we interchange $z_i$ and $z_j$. Hence, for any random variable $x$, 
    \begin{align*}
        \E[|x + z_i - z_j|] = \frac{1}{2} (\E[|x + z_i - z_j|] + \E[|x + z_i - z_j|])
        \ge \E[|x|].
    \end{align*}

    We shall further apply this symmetry trick in the following to cancel terms with opposite signs.
    Let $p$ denote the number of $+1$'s and $q$ denote the number of $-1$'s in $w$, such that $p + q = n$. 
    If $p \neq q$, by the symmetry trick,
    \begin{align*}
        \E\big[\big|d \cdot v_{d+1} + \nlsum_{i=1}^{d} w_i z_i^2 \big|\big] 
        \ge\; \E\big[\big| d \cdot v_{d+1} + \nlsum_{i=1}^{|p-q|} z_i^2 \big|\big]
        \ge \Var\big(\nlsum_{i=1}^{|p-q|} z_i^2\big) 
        =\; 2 |p - q| \ge 2.  
    \end{align*}
    If $p = q$, again, the symmetry trick yields 
    \begin{align*}
        \E[| d \cdot v_{d+1} + \nlsum_{i=1}^{d} w_i z_i^2 |] 
        \ge \E[ |d \cdot v_{d+1} + z_1^2 - z_2^2| ] 
        \ge \Var\big(z_1^2 - z_2^2\big) = 4.
    \end{align*}
    Hence, regardless of $p$ and $q$, we have $\inf_{w \in \{\pm 1\}^d} \E[| d \cdot v_{d+1} + \nlsum_{i=1}^{d} w_i z_i^2 |] \ge 2$, which \mbox{further yields}
    \begin{align*}
        \E\Big[\Big|v_{d+1} + \nlsum_{i=1}^{d} v_i z_i^2 \Big|\Big] \ge 2 \cdot \frac{(1 - v_{d+1}^2)^{1/2}}{d}.
    \end{align*}
    Since $|v_{d+1}| < 2\sqrt{d/(4d + 1)}$, 
    \begin{align*}
        \E\Big[\Big|v_{d+1} + \nlsum_{i=1}^{d} v_i z_i^2 \Big|\Big]
        = 2 \cdot \frac{1}{\sqrt{4d + 1} \cdot d}
        = 0.8 d^{-3/2}.
    \end{align*}
    Hence, overall we have 
    \begin{align*}
        \inf\nolimits_{v \in \s^{d}}\E\Big[\Big|v_{d+1} + \nlsum_{i=1}^{d} v_i z_i^2\Big|\Big] \ge 0.8 d^{-3/2},
    \end{align*}
    which completes the proof. 
\end{proof}

Based on Lemma~\ref{lem:sq-g-lb}, we can prove the more general Lemma~\ref{lem:exp-abs-lb} below.

\begin{lemma}  \label{lem:exp-abs-lb}
    Let $x$ be a $d$-dimensional zero-mean Gaussian random vector with covariance $\Sigma$. For any $d \times d$ symmetric matrix $A$ and constant $b \in \R$ that satisfy $\|A\|_F^2 + b^2 = 1$, there exists an absolute constant $a > 0$, such that $\E[|x^{\top} A x + b|] \ge a \lambda_{\min}(\Sigma) d^{-3/2}$.
\end{lemma}

\begin{proof}
    Let $y := \Sigma^{-1/2} x$. Then $y$ is a standard Gaussian random vector, and $x^{\top} A x = y^{\top} \Sigma^{1/2} A \Sigma^{1/2} y$. 
    Let $U^{\top} \Lambda U$ be the eigenvalue decomposition of $\Sigma^{1/2} A \Sigma^{1/2}$. Then, 
    \begin{align*}
        x^{\top} A x = y^{\top}U^{\top} \Lambda U y = z^{\top} \Lambda z,
    \end{align*}
    where $z := Uy$ is still a standard Gaussian random vector.

    By the unitary invariance of the Frobenius norm,
    \begin{align*}
        \|\Lambda\|_{F} = \|U^{\top} \Lambda U\|_F =\;& \|\Sigma^{1/2} A \Sigma^{1/2}\|_F \ge \lambda_{\min}(\Sigma) \|A\|_F.
    \end{align*}
    Hence, 
    \begin{align*}
        \|\Lambda\|_F^2 + b^2 \ge \lambda_{\min}^2(\Sigma) \|A\|_F^2 + b^2 
        \ge \lambda_{\min}^2(\Sigma) \wedge 1.
    \end{align*}
    Therefore, by Lemma~\ref{lem:sq-g-lb}, there exists an absolute constant $a > 0$, such that 
    \begin{align*}
        \inf\nolimits_{\|A\|_F^2 + b^2 = 1} \E[|x^{\top} A x + b|]
        \ge \inf\nolimits_{\|\Lambda\|_F^2 + b^2 \ge \lambda_{\min}^2(\Sigma) \wedge 1} \E[|z^{\top} \Lambda z + b|]
        \ge a (\lambda_{\min}(\Sigma) \wedge 1) d^{-3/2},
    \end{align*}
    which completes the proof. 
\end{proof}

\subsection{Quadratic regression bound}
\label{sec:qr-bound}

The following quadratic regression bound is at the core of proving Theorem~\ref{thm:main-poly}.
Its proof builds on the new persistency of excitation result (Lemma~\ref{lem:per-exc}).
We retain $(e_t)_{t\ge 1}$ in the bound, as in our problem $(e_t)_{t\ge 1}$ may not correspond to a martingale, and may contain an additional small error term resulting from using $\sM h_t$ to approximate $\sz_t$.
For notational convenience, we note that the $h_t, c_t, \cF_t$ in Lemma~\ref{lem:qr} and its proof slightly abuse the notation, which uses different variables from the rest of the paper. Hence, the indices start with $t = 1$, rather than $t = H$ as in the \corele{} and \coreli{} algorithms.

\begin{lemma}  \label{lem:qr}
    Let $(\sh_t)_{t\ge 1}$ be a sequence of $d$-dimensional Gaussian random vectors adapted  to the filtration $(\cF_t)_{t\ge 1}$ with $\|\E[\sh_t (\sh_t)^{\top}]\|_2^{1/2} \le \sigma$ for all $t \ge 1$. 
    Define random variable $c_t := (\sh_t)^{\top} \sN \sh_t + \starb + e_t$, where $\sN \in \R^{d\times d}$ is a positive semidefinite matrix and $\starb \in \R$ is a constant. 
    Assume $\sigma$ and $\|\sN\|_2$ are $\bigo(1)$.
    Define $h_t = \sh_t + \delta_t$, where the perturbation vector $\delta_t$ can be correlated with $\sh_t$ and its $\ell_2$-norm is sub-Gaussian with $\E[\|\delta_t\|] \le \epsilon$, $\|\|\delta_t\|\|_{\psi_2} \le \epsilon$. 
    Define $\starf_t := \svec(\sh_t (\sh_t)^{\top})$ and $\sof_t := [\starf_t; 1]$.
    Assume that $(\sof_t)_{t\ge 1}$ satisfies $\lambda_{\min}(\sum_{t=1}^{T} \sof_t (\sof_t)^{\top}) \ge \beta_d^2 T$ for $\beta_d > 0$.
    Consider 
    \begin{align} 
        (\eN, \eb) \in \argmin_{N = N^{\top}, b} \nlsum_{t=1}^{T} (c_t - \|h_t\|_{N}^2 - b)^2.  \label{eqn:qr}
    \end{align} 
    There exsits an absolute constant $a > 0$, such that for a given $p \in (0, 1)$, under the condition that $\epsilon \le \min(\sigma d^{1/2}, a \beta_d \sigma^{-1} d^{-1/2} (\log(T/p))^{-1})$,
    with probability at least $1 - p$,
    \begin{align*}
        \|\eN - \sN\|_F =\;& \bigo\Big(\epsilon \sigma \beta_d^{-1} d^{1/2} \log(T/p) \\
        &\quad + \epsilon \sigma \beta_d^{-2} d^{1/2} T^{-1} \log(T/p) \nlsum_{t=1}^{T} \|e_t\|
        + \beta_d^{-2} T^{-1} \Big\|\nlsum_{t=1}^{T} \sof_t e_t\Big\| \Big).
    \end{align*}
\end{lemma}

\begin{proof}
    Regression~\eqref{eqn:qr} can be written as
    \begin{align*}
        \argmin_{\svec(N), b} \nlsum_{t=1}^{T} \big(c_t - \svec(h_t h_t^{\top})^{\top} \svec(N) - b\big)^2.
    \end{align*}
    Define $f_t := \svec(h_{t} h_t^{\top})$ and $\of_t := [f_t; 1]$.
    It is a linear regression problem with extended covariates $\of_t$, which can be further rewritten as
    \begin{align}  \label{eqn:qr-svec}
        \argmin_{\svec(N), b} \nlsum_{t=1}^{T} \big(c_t - \of_t^{\top} [\svec(N); b]\big)^2.
    \end{align}
    Let $\oF := [\of_1, \of_{2}, \ldots, \of_{T} ]^{\top}$ be the $T \times \frac{d^2+d+2}{2}$  matrix whose $t$th row is $f_{t}^{\top}$. 
    Define $\soF$ similarly by replacing $\of_t$ by $\sof_{t}$. Solving linear regression~\eqref{eqn:qr-svec} gives
    \begin{align*}
        \oF^{\top} \oF [\svec(\eN); \eb] = \nlsum_{t=1}^{T} \of_t c_t.
    \end{align*}
    Substituting $c_t = (\sof_t)^{\top} [\svec{(\sN)}; \starb] + e_t$  into the above equation yields 
    \begin{align*}
        \oF^{\top} \oF [\svec(\eN); \eb] = \oF^{\top} \soF [\svec(\sN); \starb] + \oF^{\top} \xi, 
    \end{align*}
    where $\xi$ denotes the vector whose $t$th element is $e_t$. Rearranging the terms, we have 
    \begin{equation}  \label{eqn:ff-nn}
        \begin{aligned}
            \oF^{\top} \oF [\svec(\eN - \sN); \eb - \starb]
            = \oF^{\top} (\soF - \oF) [\svec(\sN); \starb] + \oF^{\top} \xi.
        \end{aligned}
    \end{equation}

    Next, we show that $\oF^{\top} \oF$ is invertible with high probability. By our assumption,
    \begin{align*}
        \lambda_{\min}((\soF)^{\top} \soF) = \lambda_{\min}\Big(\nlsum_{t=1}^{T} \of_t (\of_t)^{\top}\Big) \ge \beta_d^2 T.
    \end{align*}
    By Weyl's inequality for singular values, 
    \begin{align*}
        |\sigma_{\min}(\oF) - \sigma_{\min}(\soF)| \le \|\oF - \soF\|_2 = \|F - \sF\|_2.
    \end{align*}
    Hence, we want to bound $\|\sF - F\|_2$, which satisfies 
    \begin{align*}
        \|\sF - F\|_2^2 \le\;& \|\sF - F\|_F^2
        = \nlsum_{t=1}^{T} \|\sh_t (\sh_t)^{\top} - h_t h_t^{\top}\|_{F}^2.
    \end{align*}
    Since $\sh_t (\sh)_t^{\top} - h_t h_t^{\top}$ has at most rank two, we have
    \begin{align*}
        \|\sh_t (\sh_t)^{\top} - h_t h_t^{\top} \|_F
        \le\;& \sqrt{2} \|\sh_t (\sh_t)^{\top} - h_t h_t^{\top} \|_2 \\
        =\;& \sqrt{2} \| \sh_t (\sh_t - h_t)^{\top} + (\sh_t - h_t) h_t^{\top} \|_2 \\
        \le\;& \sqrt{2} (\|\sh_t\| + \|h_t\|) \|\delta_t\|.
    \end{align*}
    Since $\sh_t$ is Gaussian with $\|\E[\sh_t (\sh_t)^{\top}]\|_2^{1/2} \le \sigma$, $\|\sh\|$ is sub-Gaussian with its mean and sub-Gaussian norm bounded by $\bigo(\sigma d^{1/2})$. Since $\|\delta_t\|$ is sub-Gaussian with its mean and sub-Gaussian norm bounded by $\epsilon \le \sigma d^{1/2}$, we conclude that $\|\sh_t (\sh_t)^{\top} - h_t h_t^{\top} \|_2$ is subexponential with its mean and subexponential norm bounded by $\bigo(\epsilon \sigma d^{1/2})$. Hence, with probability at least $1 - p$, 
    \begin{align*}
        \|\sh_t (\sh_t)^{\top} - h_t h_t^{\top} \|_F = \bigo(\epsilon \sigma d^{1/2} \log(T/p)).
    \end{align*}
    Therefore, 
    \begin{align*}
        \| \sF - F \|_F^2 = \nlsum_{t=1}^{T} \|\sh_t (\sh_t)^{\top} - h_t h_t^{\top}\|_{F}^2
        = \bigo(\epsilon^2 \sigma^2 d T \log^2(T/p) ).
    \end{align*}
    It follows that 
    \begin{align*}
        \|\sF - F\|_2 = \bigo(\epsilon \sigma (d T)^{1/2} \log(T/p)).
    \end{align*}
    Hence, there exists an absolute constant $a > 0$, such that as long as $\epsilon \le a \beta_d \sigma^{-1} d^{-1/2} (\log(T/p))^{-1}$, we have 
    \begin{align*}
        |\sigma_{\min}(F) - \sigma_{\min}(\sF)| \le \beta_d T^{1/2} / 2.
    \end{align*}
    It follows that 
    \begin{align*}
        \lambda_{\min}(\oF^{\top} \oF) = \sigma_{\min}^2(\oF) = \Omega( \beta_d T ).
    \end{align*}

    Now we return to~\eqref{eqn:ff-nn}. By inverting $\oF^{\top} \oF$, we obtain 
    \begin{equation} \label{eqn:a_b_def}
        \begin{aligned}
            \|[\svec(\eN - \sN); \eb - \starb]\|
            =\;& \|\oF^{\dagger} (\soF - \oF) [\svec(\sN); \starb] + \oF^{\dagger} \xi\| \\
            \le\;& \underbrace{\|\oF^{\dagger} (\soF - \oF) [\svec(\sN); \starb]\|}_{(a)} + \underbrace{\|\oF^{\dagger} \xi\|}_{(b)}.
        \end{aligned}
    \end{equation}
    Term $(a)$ is upper bounded by
    \begin{align*}
        \sigma_{\min}(\oF)^{-1} \|(\soF - \oF) [\svec(\sN); \starb]\|
        =\;& \bigo(\sigma_{\min}(\oF)^{-1}) \|(\sF - F) \svec(\sN)\| \\
        =\;& \bigo(\beta_d^{-1} T^{-1/2}) \|(\sF - F) \svec(\sN)\|.
    \end{align*}
    Using arguments similar to those in~\citep[Section B.2.13]{mhammedi2020learning}, \mbox{we have}
    \begin{align*}
        \|(\sF - F) \svec(\sN)\|^2
        =\;& \nlsum_{t=1}^{T} \ip{\svec(\sh_t (\sh_t)^{\top}) - \svec(h_t h_t^{\top})}{\svec(\sN)}^2 \\
        =\;& \nlsum_{t=1}^{T} \ip{\sh_t (\sh_t)^{\top} - h_t h_t^{\top}}{\sN}_F^2 \\
        \le\;& \|\sN\|_2^2 \nlsum_{t=1}^{T} \|\sh_t (\sh_t)^{\top} - h_t h_t^{\top}\|_{\ast}^2 \\
        \overset{(i)}{\le}\;& 2 \|\sN\|_2^2 \nlsum_{t=1}^{T} \|\sh_t (\sh_t)^{\top} - h_t h_t^{\top}\|_{F}^2 \\
        =\;& 2 \|\sN\|_2^2 \|\oF - F\|_F^2,
    \end{align*}
    where {$\|\cdot\|_{\ast}$ denotes the nuclear norm,} $(i)$ follows from the fact that the matrix $\sh_t (\sh_t)^{\top} - h_t h_t^{\top}$ has at most rank two. 
    Hence, term $(a)$ in~\eqref{eqn:a_b_def} is bounded by 
    \begin{align*}
        \bigo\big(\beta_d^{-1} T^{-1/2} \epsilon \sigma \|\sN\|_2 d^{1/2} T^{1/2} \log(T/p)\big)
        = \bigo( \beta_d^{-1} d^{1/2} \epsilon \sigma \log(T/p)).
    \end{align*}

    Now we consider term $(b)$ in \eqref{eqn:a_b_def}:
    \begin{align*}
        (b) = \|\oF^{\dagger} \xi\| \le \lambda_{\min}(\oF^{\top} \oF)^{-1} \|\oF^{\top} \xi\|
        = \bigo(\beta_d^{-2} T^{-1}) \Big\|\nlsum_{t=1}^{T} \of_t e_t\Big\|.
    \end{align*}
    Since 
    \begin{align*}
        \Big\|\nlsum_{t=1}^{T} \of_t e_t\Big\|
        \le \Big\|\nlsum_{t=1}^{T} \sof_t e_t\Big\| + \nlsum_{t=1}^{T}\|\of_t - \sof_t\| \|e_t\|.
    \end{align*}
    we have 
    \begin{align*}
        (b) = \bigo\Big(\beta_d^{-2} T^{-1} \Big\|\nlsum_{t=1}^{T} \sof_t e_t\Big\|
        + \epsilon \sigma \beta_d^{-2} d^{1/2} T^{-1} \log(T/p) \nlsum_{t=1}^{T}\|e_t\| \Big).
    \end{align*}
    Combining the bounds on $(a)$ and $(b)$, we show that with probability at least $1 - p$, 
    \begin{align*}
        &\|[\svec(\eN - \sN); \eb - \starb]\| \\
        =\;& \bigo( \epsilon \sigma \beta_d^{-1} d^{1/2} \log(T/p)
        + \epsilon \sigma \beta_d^{-2} d^{1/2} T^{-1} \log(T/p) \nlsum_{t=1}^{T} \|e_t\|
        + \beta_d^{-2} T^{-1} \Big\|\nlsum_{t=1}^{T} \of_t e_t \Big\| ),
    \end{align*}
    completing the proof.
\end{proof}

\subsection{Perturbed linear regression bound}

Identifying the time-invariant latent dynamics involves linear regression with \emph{correlated data} and \emph{perturbed measurements}. The following Lemma~\ref{lem:pert-lr} extends the previous linear system identification result in~\citep{simchowitz2018learning} to the case with noises in both input and output variables.
In Lemma~\ref{lem:pert-lr}, $\gamma$ and $q$ are treated as dimension-free constants (in contrast to Lemma~\ref{lem:qr}), which is indeed the case in our application of Lemma~\ref{lem:pert-lr} to $(\sz_t)_{t\ge H}$ in analyzing \sysid{}~\eqref{eqn:sys-id} for~\corele{} and in analyzing the alignment matrix estimation~\eqref{eqn:align-lr} in Algorithm~\ref{alg:cosysid} for \coreli{} in \S\ref{sec:main-proof}.
Note that the bound in Lemma~\ref{lem:pert-lr} is worse than that in the time-varying setting in Part I of this work, due to the treatment of correlated data.

\begin{lemma}  \label{lem:pert-lr}
    Let $(\sx_t)_{t\ge 1}$ be a sequence of $d_1$-dimensional Gaussian random vectors adapted to the filtration $(\cF_t)_{t\ge 1}$
    with $\|\E[\sx_t (\sx_t)^{\top}]\|_2^{1/2} \le \sigma$ for all $t \ge 1$.
    Define $\sy_t = \sA \sx_t + e_t$, where $\sA \in \R^{d_2 \times d_1}$ and $e_t \given \cF_t$ is Gaussian with zero mean and $\|\E[e_t e_t^{\top}]\|_2^{1/2} \le \epsilon$.
    Define $y_t = \sy_t + \delta_{t}^{y}$ and $x_t = \sx_t + \delta_{t}^{x}$, where the perturbation vectors $\delta_t^x$ and $\delta_t^y$ can be correlated with $\sx_t$ and $\sy_t$, and their $\ell_2$-norms are sub-Gaussian with $\E[\|\delta_t^x\|] \le \epsilon_x$, $\|\|\delta_t^x\|\|_{\psi_2} \le \epsilon_x$ and $\E[\|\delta_t^y\|] \le \epsilon_y$, $\|\|\delta_t^y\|\|_{\psi_2} \le \epsilon_y$.
    Assume that $(\sx_t)_{t\ge 1}$ satisfies the $(k, \gamma^2 I, q)$-BMSB condition, and that $\|\sA\|_2, \sigma, \epsilon$ are $\bigo(1)$ and $k, \gamma, q$ are $\Theta(1)$.
    Consider
    \begin{align}  \label{eqn:pert-lr}
        \eA \in \argmin_{A\in \R^{d_2 \times d_1}} \nlsum_{t=1}^{T} \|y_t - A x_t\|^2.
    \end{align}
    Then, there exist absolute constants $a_0, a_1 > 0$, such that for a given $p \in (0, 1)$, under the condition that $T \ge a_0 k q^{-2} (\log(1/p) + d_1 \log(10/q) + d_1 \log(\sigma \gamma^{-1} d_1 \log(T/p)) + d_2)$, $\epsilon_x \le \min(d_1^{-1/2} d_2^{-1/2} (\log(T/p))^{-3/2}, \allowbreak a_1 \gamma q (\log(T/p)))$ and $\epsilon_y \le d_1^{-1/2} (\log(T/p))^{-1}$,
    with probability at least $1 - p$, 
    \begin{align*}
        \|\eA - \sA\|_2
        = \bigo((\epsilon_x + \epsilon_y)d_1^{1/2} \log(T / p)
        + \epsilon (d_2 + d_1 \log(d_1 \log(T / p)) + \log(1/p))^{1/2} T^{-1/2}).
    \end{align*}
\end{lemma}

\begin{proof} 
    Let $X \in \R^{T \times d_1}$ denote the matrix whose $t$th row is $x_t^{\top}$. Define $\sX, Y, E, \Delta_x, \Delta_y$ similarly. 
    To solve the regression problem, we set the gradient of the objective to be zero and substitute in $Y = \sX (\sA)^{\top} + E + \Delta_y$ to obtain 
    \begin{align}  \label{eqn:ea}
        \eA (X^{\top} X) = \sA (\sX)^{\top} X + E^{\top} X + \Delta_y^{\top} X.
    \end{align}
    Substituting in $X = \sX + \Delta_x$ gives 
    \begin{equation}  \label{eqn:pert-lr-error}
        \begin{aligned}
            &(\eA - \sA)((\sX)^{\top} \sX) \\
            =\;& \sA (\sX)^{\top} \Delta_x - \eA (\Delta_x^{\top} \Delta_x + \Delta_x^{\top} \sX + (\sX)^{\top} \Delta_x)
            + E^{\top} \sX + E^{\top} \Delta_x + \Delta_y^{\top} \sX + \Delta_y^{\top} \Delta_x.
        \end{aligned}
    \end{equation}
    Now we deal with each term on the right-hand side.
    Since $(\sx_t)_{t\ge 1}$ are Gaussian, by the tail bound of the sub-Gaussian random variable $\|\sx_t\|$ and the union bound, with probability at least $1 - p$, $\|\sx_t\| = \bigo(\sigma d_1^{1/2} \log^{1/2}(T/p))$ for $1\le t\le T$.  
    By the triangle inequality, $\|(\sX)^{\top} \sX\|_2 \le \nlsum_{t=1}^{T} \|\sx_t (\sx_t)^{\top}\|_2 = \nlsum_{t=1}^{T} \|\sx_t\|^2$, and thus 
    \begin{align*}
        \|\sX\|_2 = \bigo(\sigma d_1^{1/2} T^{1/2} \log^{1/2}(T/p)).
    \end{align*}
    Similarly, with probability at least $1 - p$,
    \begin{align*}
        \|E\|_2 = \bigo(\epsilon d_2^{1/2} T^{1/2} \log^{1/2}(T/p)).
    \end{align*}
    Such arguments also apply to $\Delta_x, \Delta_y$. Since $(\|\delta_t^{x}\|)_{t\ge 1}$ are sub-Gaussian with $\E[\|\delta_t^x\|]\le \epsilon_x$ and $\|\|\delta_t^x\|\|_{\psi_2} \le \epsilon_x$, with probability at least $1 - p$, $\|\delta_t^x\| = \bigo(\epsilon_x \log^{1/2}(T/p))$; similarly, $\|\delta_t^y\| = \bigo(\epsilon_y \log^{1/2}(T/p))$.
    Therefore, we further have
    \begin{align*}
        \|\Delta_x\|_2 = \bigo(\epsilon_x T^{1/2} \log^{1/2}(T/p)), \quad  
        \|\Delta_y\|_2 = \bigo(\epsilon_y T^{1/2} \log^{1/2}(T/p)).
    \end{align*}
    It remains to bound $\|\eA\|_2$.  
    By~\citep[Appendix D]{simchowitz2018learning}, there exist absolute constants $a_0, a_2 > 0$, such that as long as $T \ge T_0 := a_0 k q^{-2} (\log(1/p) + d_1 \log(10/q) + d_1 \log(\sigma \gamma^{-1} d_1 \log(T/p)))$, with probability at least $1 - p$, $\lambda_{\min}((\sX)^{\top}\sX) \ge a_2^2 \gamma^2 q^2 T$. 
    By Weyl's inequality for singular values, we have  
    \begin{align*}
        |\sigma_{\min}(X) - \sigma_{\min}(\sX)| \le \|\Delta_x\|_2.
    \end{align*}
    Hence, there exists an absolute constant $a_1 > 0$, such that as long as $\epsilon_x \le a_1 \gamma q (\log(T/p))^{-1/2}$, we have $|\sigma_{\min}(X) - \sigma_{\min}(\sX)| \le a_2 \gamma q T^{1/2} / 2$. Since $\sigma_{\min}(\sX) \ge a_2 \gamma q T^{1/2}$, 
    \begin{align*}
        \lambda_{\min}(X^{\top} X) = \sigma_{\min}^2(X) = \Omega(\gamma^2 q^2 T).
    \end{align*}
    Hence, we can invert $X^{\top} X$ in~\eqref{eqn:ea} and obtain 
    \begin{align*}
        \eA =\;& (\sA (\sX)^{\top} + E^{\top} + \Delta_y^{\top}) X (X^{\top} X)^{-1} \\
        =\;& \sA (X^{\dagger} \sX)^{\top} + (X^{\dagger} E)^{\top} + (X^{\dagger} \Delta_y)^{\top} \\
        =\;& \sA - \sA (X^{\dagger} \Delta_x)^{\top} + (X^{\dagger} E)^{\top} + (X^{\dagger} \Delta_y)^{\top}.
    \end{align*}
    Then, we have 
    \begin{align*}
        \|\eA\|_2 =\;& \|\sA\|_2 + (\|\sA\|_2 \|\Delta_x\|_2 + \|E\|_2 + \|\Delta_y\|_2) \|X^{\dagger}\|_2 \\
        =\;& \bigo((\gamma q)^{-1} (\epsilon_x + \epsilon d_2^{1/2} + \epsilon_y) \log^{1/2}(T/p) ) \\
        =\;& \bigo(d_2^{1/2} \log^{1/2}(T/p)).
    \end{align*}
    By~\citep[Theorem 2.4]{simchowitz2018learning}, as long as $T \ge T_0$, 
    \begin{align*}
        \|E^{\top} (\sX)^{\dagger}\|_2
        = \bigo(\epsilon (d_2 + d_1 \log(d_1 \log(T / p)) + \log(1/p))^{1/2} T^{-1/2}).
    \end{align*}
    Combining all the above individual bounds for the terms on the right-hand side of~\eqref{eqn:pert-lr-error}, \mbox{we have}
    \begin{align*}
        \|\eA - \sA\|_2
        =\;& \bigo(\epsilon_x d_1^{1/2} d_2^{1/2} \log^{3/2}(T/p) + \epsilon_y d_1^{1/2}\log(T / p) \\
        &\quad + \epsilon (d_2 + d_1 \log(d_1 \log(T / p)) + \log(1/p))^{1/2} T^{-1/2}).
    \end{align*}
    By the assumptions on $\epsilon_x, \epsilon_y, T$, we have $\|\eA\|_2 = \bigo(1)$, which, combined with the above individual bounds for the terms on the right-hand side of~\eqref{eqn:pert-lr-error}, strengthens the bound on $\|\eA - \sA\|_2$ and completes the proof. 
\end{proof}

\subsection{Stable linear system under small perturbations}

To quantify the impact of the truncation error $\delta_t$ on the state covariance and the cost, we introduce the following lemma that bounds the state covariance difference for a stable linear system under small perturbations. 

\begin{lemma}  \label{lem:ls-pert}
    Consider a linear system $x_{t+1} = \sA x_t + w_t + \delta_t$, $t \ge 0$, where $x_0$ is a zero-mean Gaussian random vector and $(w_t)_{t\ge 0}$ are i.i.d. sampled from $\gauss(0, \Sigma_w)$ with $w_t$ independent of $x_t$ for each $t\geq 0$. 
    For $t \ge 0$, perturbation $\delta_t$ is a zero-mean Gaussian random vector that is independent of $w_t$ and may be correlated with $x_t$; let $\Sigma_t$ denote the covariance matrices of $x_t$.
    Assume that the operator norms of $\sA$, $\Sigma_w$, and $\Cov(x_0)$ are $\bigo(1)$ and $\rho(\sA) < 1$; assume that for $t \ge 0$, $\|\Cov(\delta_t)\|_2^{1/2} \le \eps (1 + \max_{\tau \le t}\|\Sigma_t\|_2^{1/2})$, where $\eps \in (0, 1)$ and satisfies $2 \eps \alpha^2(\sA) (1 + 2\|\sA\|_2) \le 1 - \rho^2(\sA)$.
    Let $\sSigma$ denote the covariance matrix of the stationary distribution of $(x_t)_{t\geq 0}$   without perturbations. 
    Recall that for square matrix $A$, we define $\alpha(A) := \sup_{k \ge 0} \|(A)^k\|_2 \rho(A)^{-k}$.
    Then,
    for large enough $t$ such that $\alpha^2(\sA) \rho^{2t}(\sA) \|\Sigma_0 - \sSigma\|_2 = \bigo(\eps)$, we have $\|\Sigma_t - \sSigma\|_2 = \bigo(\eps)$.
\end{lemma}

\begin{proof}
    By the definition of the covariance matrix,
    \begin{align*}
        \sSigma =\;& \sA \sSigma (\sA)^{\top} + \Sigma_w, \\
        \Sigma_{t+1} =\;& \sA \Sigma_t (\sA)^{\top} + \Sigma_w + \Cov(\delta_t) + \sA \Cov(x_t, \delta_t) + \Cov(\delta_t, x_t) (\sA)^{\top}.
    \end{align*}
    Hence, taking the difference and defining the matrix $\Delta_t := \Cov(\delta_t) + \sA \Cov(x_t, \delta_t) + \Cov(\delta_t, x_t) (\sA)^{\top}$, we have 
    \begin{align*}
        \Sigma_{t+1} - \sSigma_{t+1} =\;& \sA (\Sigma_t - \sSigma_t) (\sA)^{\top} + \Cov(\delta_t) + \sA \Cov(x_t, \delta_t) + \Cov(\delta_t, x_t) (\sA)^{\top} \\
        =\;& \sA (\Sigma_t - \sSigma_t) (\sA)^{\top} + \Delta_t.
    \end{align*}
    Moreover, the matrix $\Delta_t$ satisfies
    \begin{align*}
        \|\Delta_t\|_2 \le\;& \|\Cov(\delta_t)\|_2 + 2\|\sA\|_2 \|\Cov(x_t, \delta_t)\|_2 \\
        \le\;& \|\Cov(\delta_t)\|_2 + 2\|\sA\|_2 \|\Sigma_t\|_2^{1/2} \|\Cov(\delta_t)\|_2^{1/2} \\
        \overset{(i)}{=}\;& \eps (1 + 2 \|\sA\|_2) (1 + \nlmax_{\tau \le t}\|\Sigma_{\tau}\|_2),
    \end{align*}
    where $(i)$ is due to $\eps < 1$.

    Now we show that $\|\Sigma_t\|_2$ is  uniformly bounded for all $t \ge 0$ by induction.
    Suppose that $\|\Sigma_{\tau}\|_2 \le a$ for all $\tau \le t$ for a constant $a > 0$ to be specified shortly.
    Then, for all $\tau \le t$, we have
    \begin{align*}
        \|\Delta_{\tau}\|_2 \le \eps (1 + 2 \|\sA\|_2) (1 + a).
    \end{align*}    
    Since
    \begin{align*}
        \Sigma_{t+1} - \sSigma
        = (\sA)^{t+1} (\Sigma_0 - \sSigma) ((\sA)^{t+1})^{\top}
        + \nlsum_{\tau=0}^{t} (\sA)^{\tau} \Delta_{t-\tau} ((\sA)^{\tau})^{\top},
    \end{align*}
    we have 
    \begin{align*}
        \|\Sigma_{t+1} - \sSigma\|_2
        \le\;& \alpha^2(\sA) \rho^{2(t+1)}(\sA) \|\Sigma_0 - \sSigma\|_2 + \nlsum_{\tau=0}^{t} \alpha^2(\sA) \rho^{2\tau}(\sA) \|\Delta_{t-\tau}\|_2 \\
        \le\;& \alpha^2(\sA) \rho^{2(t+1)}(\sA) \|\Sigma_0 - \sSigma\|_2 + \frac{\eps \alpha^2(\sA) (1 + 2 \|\sA\|_2) (1 + a)}{1 - \rho^{2}(\sA)}.
    \end{align*}
    As long as $\eps \alpha^2(\sA) (1 + 2\|\sA\|_2) \le \frac{1 - \rho^2(\sA)}{2}$, we can choose 
    \begin{align*}
        a := 2 (\|\sSigma\|_2 + \alpha^2(\sA) \|\Sigma_0 - \sSigma\|_2) + 1 = \bigo(1),
    \end{align*}
    such that $\|\Sigma_{t+1}\|_2 \le \|\sSigma\|_2 + \|\Sigma_{t+1} - \sSigma\|_2 \le a$, completing the induction.
    Hence, for all $t \ge 0$, $\|\Sigma_t\|_2 \le a$, and 
    \begin{align*}
        \|\Sigma_{t} - \sSigma\|_2
        \le \alpha^2(\sA) \rho^{2t}(\sA) \|\Sigma_0 - \sSigma\|_2
        + \frac{\eps \alpha^2(\sA) (1 + 2 \|\sA\|_2) (1 + a)}{1 - \rho^{2}(\sA)}.
    \end{align*}
    Therefore, for large enough $t$ such that $\alpha^2(\sA) \rho^{2t}(\sA) \|\Sigma_0 - \sSigma\|_2 = \bigo(\eps)$, we have $\|\Sigma_t - \sSigma\|_2 = \bigo(\eps)$, which completes the proof.
\end{proof}

\subsection{Proof of Theorem~\ref{thm:main-poly}}
\label{sec:main-proof}

In this section, we prove the sample complexity bounds for \corele{} and \coreli{} in Theorem~\ref{thm:main-poly}. Without loss of generality, we assume the system is expressed in the normalized parameterization where $\soQ := \nlsum_{t=0}^{d_x-1} ((\sA)^{t})^{\top} \sQ (\sA)^{t} = I$, since otherwise we can transform it with an invertible matrix to satisfy this condition.

\vspace{2pt}
\noindent\textbf{Learning of the state representation function.} 
By Proposition~\ref{prp:multi-step-cost}, we have
\begin{align*}
    \oc_t := \nlsum_{\tau=t}^{t+d_x-1} (c_{\tau} - \|u_{\tau}\|_{\sR}^2) = \|\sM h_t\|^2 + \odelta_t + \sob + \overe_t,
\end{align*}
where $\odelta_t = \bigo(\alpha^2\rho^H \log(T/p))$ is a small error term, $\sob = \bigo(d_x)$ is a positive constant, and $\overe_t$ is a zero-mean subexponential random variable with $\|\overe_t\|_{\psi_1} = \bigo(d_x^{3/2})$. Recall that we define $\sN := (\sM)^{\top} \sM$, $f_t := \svec(h_t h_t^{\top})$, $\of_t := [\svec(h_t h_t^{\top}); 1]$, and $d_h := H(d_y + d_u)$.
Rewriting the above equation, we have 
\begin{align}
    \oc_t = \of_t^{\top} [\svec(\sN); \sob] + \odelta_t + \overe_t. 
\end{align}
By Lemma~\ref{lem:per-exc}, for a given $p \in (0, 1)$, there exists a problem-dependent constant $a_0 > 0$, such that as long as $T \ge a_0 H d_h^8 \log(d_h/p)$, with probability at least $1 - p$, 
\begin{align*}
    \lambda_{\min}\Big(\nlsum_{t=H}^{T+H-1} \of_t (\of_t)^{\top}\Big) = \Omega( d_h^{-9} T ).
\end{align*}
Moreover, by~\citep[Lemma 12]{tian2022cost}, $\|\Cov(h_t)\|_2 = \|\E[h_t h_t^{\top}]\|_2 = \bigo(H)$. Since $\|\sM\|_2 = \bigo(1)$, $\|\sN\|_2 = \bigo(1)$.
Then, by Lemma~\ref{lem:qr} with $\epsilon = 0$ therein, $\eN$ obtained by solving regression~\eqref{eqn:alg-qr} has the guarantee that  
\begin{align*}
    \|\eN - \sN\|_F =\;& \bigo\Big(d_h^9 T^{-1} \Big\|\nlsum_{t=1}^{T} \of_t (\odelta_t + \overe_t)\Big\| \Big).
\end{align*}
By Proposition~\ref{prp:multi-step-cost}, there exists a dimension-free constant $a_1 > 0$, such that as long as $H \ge \frac{a_1 \log(\alpha T \log(T/p))}{\log(1/\rho)}$,
\begin{align*}
    \Big\|\nlsum_{t=H}^{T+H-1} \of_t \overe_t\Big\| = \bigo(d_x^{3/2} d_h H^{1/2} T^{1/2} \log^{1/2}(H/p)).
\end{align*}
Since $\|\sum_{t=H}^{T+H-1} \of_t \odelta_t\| = \bigo(\alpha^2 d_h \rho^H T \log(T/p))$,  we have 
\begin{align*}
    \|\eN - \sN\|_F = \bigo\Big(d_h^9 T^{-1} \big(d_x^{3/2} d_h H^{1/2} T^{1/2} \log^{1/2}(H/p)
    + \alpha^2 d_h \rho^H T \log(T/p)\big) \Big).
\end{align*}
Hence, there exists a dimension-free constant $a_2 > 0$, such that as long as $H \ge \frac{a_2 \log(\alpha T \log(T/p))} {\log(1/\rho)}$, $\|\eN - \sN\|_F$ is bounded by
\begin{align}  \label{eqn:n-m-bd}
    \bigo( H^{21/2} d_x^{3/2} (d_y + d_u)^{10} T^{-1/2} \log^{1/2}(H/p)).
\end{align}

By~\citep[Lemma 5.4]{tu2016low}, there exists an orthogonal matrix $S$, such that $\|\eM - S \sM\|_F$ is of the same order as $\|\eN - \sN\|_F$. To understand the approximation error $\ez_t - S\sz_t$, recall that $\sz_t = \sM h_t + \delta_t$, where $\delta_t = (\soA)^H \sz_{t-H}$. Then, 
\begin{align*}
    \|\ez_t - S\sz_t\| = \| (\eM - S \sM) h_t - S \delta_t \|
    \le \|\eM - S \sM\|_2 \|h_t\| + \|\delta_t\|.
\end{align*}
Since $\|h_t\|$ is sub-Gaussian with $\E[\|h_t\|] = \bigo(d_h^{1/2})$, $\|\|h_t\|\|_{\psi_2} = \bigo(d_h^{1/2})$, we have $\|\eM - S \sM\|_2 \|h_t\|$ is sub-Gaussian with its mean and sub-Gaussian norm bounded by
\begin{align} \label{eqn:err-subg-norm}
    \bigo(H^{11} d_x^{3/2} (d_y + d_u)^{21/2} T^{-1/2} \log^{1/2}(H/p)).
\end{align}
Notice that $\|\delta_t\|$ is sub-Gaussian with mean and sub-Gaussian norm bounded by $\bigo(\alpha(\soA) \rho(\soA)^{H} d_x^{1/2})$, which, by our choice of $H$, is dominated by~\eqref{eqn:err-subg-norm}. Hence, for all $t \ge H$, $\|\ez_t - S\sz_t\|$ is sub-Gaussian with its mean and sub-Gaussian norm bounded by~\eqref{eqn:err-subg-norm}.

\vspace{2pt}
\noindent\textbf{Identification of the latent cost matrix.}
The latent cost is also described in Proposition~\ref{prp:latent-model}, \mbox{given by} 
\begin{align*}
    c_t = \|\sz_t\|_{\sQ}^2 + \|u_t\|_{\sR}^2 + \starb + e_{t},
\end{align*}
where $\starb = \E[ \|x_t -  \sz_t\|_{\sQ}^2]$, $e_t = \|x_t - \sz_t\|_{\sQ}^2 + 2\ipc{\sz_t}{x_t - \sz_t}_{\sQ}-\starb$ is a zero-mean subexponential random variable with $\|e_t\|_{\psi_1} = \bigo(d_x^{1/2})$, and the random process $(e_t)_{t\ge H}$ is adapted to the filtration $(\cF_t)_{t\ge H}$.
In a similar way to the analysis for $([\svec(h_t h_t^{\top}); 1])_{t\ge H}$ in the proof of Lemma~\ref{lem:per-exc}, $([\svec(\sz_t (\sz_t)^{\top}); 1])_{t\ge H}$ satisfies
$\lambda_{\min}(\nlsum_{t=H}^{T+H-1} [\svec(\sz_t (\sz_t)^{\top}); 1] [\svec(\sz_t (\sz_t)^{\top}); 1]^{\top}) = \Omega(d_x^{-9} T)$, which remains true under the similarity transformation $S$.
As $\ez_t$ approximates $S \sz_t$, the ground truth for the latent cost matrix is $S \sQ S^{\top}$. 
By the perturbed quadratic regression bound (Lemma~\ref{lem:qr}) with $\epsilon$ being bounded by~\eqref{eqn:err-subg-norm} and $\sigma = \bigo(d_x^{1/2})$ therein, $\tQ$ from regression~\eqref{eqn:alg-qr-q} has the guarantee that  
\begin{align*}
    &\|\tQ - S \sQ S^{\top}\|_F \\
    =\;& \bigo\Big(H^{11} d_x^{3/2} (d_y + d_u)^{21/2} T^{-1/2} \log^{1/2}(H/p) \cdot (d_x^{9/2} d_x^{1/2} \log(T / p) \\
    &\quad + d_x^9 d_x^{1/2} \log(T/p) T^{-1} \nlsum_{t=H}^{T+H-1} \|e_t\|)
    + d_x^{9} T^{-1} \Big\|\nlsum_{t=H}^{T+H-1} [\svec(\sz_t (\sz_t)^{\top}); 1] e_t \Big\|\Big).
\end{align*}

By the tail bound of subexponential random variables and the union bound, with probability at least $1 - p$, $\|e_t\| = \bigo(d_x^{1/2} \log(T/p))$ for $H \le t\le T + H - 1$.
By a similar analysis to that for $\nlsum_{t=H}^{T+H-1} \of_t \overe_t$ in the proof of Proposition~\ref{prp:multi-step-cost}, we have
\begin{align*}
    \nlsum_{t=H}^{T+H-1} [\svec(\sz_t (\sz_t)^{\top}); 1] e_t
    = \bigo(d_x^{1/2} d_x H^{1/2} T^{1/2} \log^{1/2}(H/p) ).
\end{align*}
Since $\sQ \succcurlyeq 0$ and $\eQ$ is the projection of $\tQ$ onto the space of positive semidefinite matrices, we have
\begin{align*}
    \|\eQ - S \sQ S^{\top}\|_F \le \|\tQ - S \sQ S^{\top}\|_F
    = \bigo(H^{11} d_x^{23/2} (d_y + d_u)^{21/2} T^{-1/2} \log^{5/2}(T/p)).
\end{align*}

\vspace{2pt}
\noindent\textbf{Identification of the latent dynamics in \corele{}.} To analyze the standard system identification procedure in \corele{}, consider the latent dynamics described in Proposition~\ref{prp:latent-model}, given by $\sz_{t+1} = \sA \sz_t + \sB u_t + \sL i_{t+1}$. To apply the perturbed linear regression bound (Lemma~\ref{lem:pert-lr}), the noise term $\sL i_{t+1} \given \cF_t$ needs to be zero-mean Gaussian, which does not hold \mbox{here, since} 
\begin{align*}
    i_{t+1} =\;& y_{t+1} - \sC (\sA \sz_t + \sB u_t) \\
    =\;& \sC ((\sA x_t + \sB u_t + w_t) + v_{t+1}) - \sC (\sA \sz_t + \sB u_t) \\
    =\;& \sC \sA (x_t - \sz_t) + \sC w_t + v_{t+1},
\end{align*}
where $x_t - \sz_t$ is $\cF_t$-measurable. To solve this problem, we consider a different filtration $(\cG_t := \sigma(y_0, u_0, y_1, \ldots, u_{t-1}, y_{t}))_{t\ge 0}$ that involves only observations and actions. Then, $\sz_t$ is $\cG_t$-measurable and $\sL i_{t+1} \given \cG_t$ is zero-mean Gaussian with the operator norm of the covariance matrix bounded by $\bigo(1)$.

By~\citep[Proposition 3.1]{simchowitz2018learning}, $(\sz_t)_{t\ge H}$ satisfies the $(\Theta(1), \Theta(1), \Theta(1))$-BMSB condition, which remains true under the similarity transformation $S$. As $\ez_t$ approximates $S \sz_t$, the ground truth for the latent dynamics is $[S \sA S^{\top}, S \sB]$.
With filtration $(\cG_t)_{t\ge 0}$, by the perturbed linear regression bound (Lemma~\ref{lem:pert-lr}), for $T$ greater than a constant polynomial in the problem parameters, we have
\begin{align*}
    \|[\eA, \eB] - S [\sA S^{\top}, \sB]\|_2
    =\;& \bigo(H^{11} d_x^{3/2} (d_y + d_u)^{21/2} T^{-1/2} \log^{1/2}(H/p) \cdot d_x^{1/2} \log(T/p) ) \\
    =\;& \bigo(H^{11} d_x^2 (d_y + d_u)^{21/2} T^{-1/2} \log^{3/2}(T/p)),
\end{align*}
where the second term in the bound in Lemma~\ref{lem:pert-lr} is dominated by the first term and omitted above.
Hence, $\|\eA - S \sA S^{\top}\|_2$, $\|\eB - S\sB\|_2$ and $\|\eQ - S \sQ S^{\top}\|_2$ are all bounded by 
\begin{align*}
    \bigo(H^{11} d_x^{23/2} (d_y + d_u)^{21/2} T^{-1/2} \log^{5/2}(T/p)).
\end{align*}

\vspace{2pt}
\noindent\textbf{Identification of the latent dynamics in \coreli{}.}
To analyze the cost-driven system identification (Algorithm~\ref{alg:cosysid}) in \coreli{}, define $\sM_1 := [\sA \sM, \sB]$ as the composition of one-step transition and representation functions and $\sN_1 := (\sM_1)^{\top} \sM_1$, which is estimated by $\eN_1$ in~\eqref{eqn:dy-qr}.
By the same analysis as that of $\eN$, we have 
\begin{align*}
    \|\eN_1 - \sN_1\|_F
    = \bigo(H^{1/2} (H (d_y + d_u) + d_x)^{10} d_x^{3/2} T^{-1/2} \log^{1/2}(H/p)).
\end{align*}
By~\citep[Lemma 5.4]{tu2016low}, there exists an orthogonal matrix  $S_1$, such that $\|\eM_1 - S_1 \sM_1\|_F$ is of the same order as $\|\eN_1 - \sN_1\|_F$. The bound on $\|\eM_1 - S_1 \sM_1\|_F$ applies to both $\|\tM - S_1 \sA \sM\|_2$ and $\|\tB - S_1 \sB\|_2$. By Algorithm~\ref{alg:cosysid},
\begin{align*}
    \tA 
    =\;& \tM \eM^{\dagger} \\
    =\;& (S_1 \sA \sM + \tM - S_1 \sA \sM) ((\sM)^{\dagger}S^{\top} + \eM^{\dagger} - (\sM)^{\dagger}S^{\top}) \\
    =\;& S_1 \sA S^{\top} + S_1 \sA \sM (\eM^{\dagger} - (\sM)^{\dagger}S^{\top}) + (\tM - S_1 \sA \sM) (\sM)^{\dagger}S^{\top} \\
    &\quad + (\tM - S_1 \sA \sM) (\eM^{\dagger} - (\sM)^{\dagger}S^{\top}).
\end{align*}
By the perturbation bound of the Moore-Penrose \mbox{inverse~\citep{wedin1973perturbation}}, $\|\eM^{\dagger} - (\sM)^{\dagger} S^{\top}\|_2 = \bigo(\|\eM - S \sM\|_2) = \bigo(\|\eN_1 - \sN_1\|_2)$. Hence, $\|\tA - S_1 \sA S^{\top}\|_2$ is of the same order as $\|\eN_1 - \sN_1\|_F$.

As mentioned for \corele{}, since $\ez_t$ approximates $S \sz_t$, the ground truth for the latent dynamics is $[S \sA S^{\top}, S \sB]$.
To align $\tA$ with $S \sA S^{\top}$, we compute another matrix $\eS_0$ by solving the regression~\eqref{eqn:align-lr} from $\eM_1 [h_t; u_t]$ to $\eM h_{t+1}$. Since $\eM_1 [h_t; u_t]$ and $\eM h_{t+1}$ approximate $S_1 \sz_{t+1}$ and $S \sz_{t+1}$, respectively,~\eqref{eqn:align-lr} is essentially a linear regression that estimates the alignment matrix $S S_1^{\top}$ with perturbed variables $\eM_1 [h_t; u_t]$ and $\eM h_{t+1}$. 
The $\ell_2$-norm of the perturbation on $S \sz_t$ is given by~\eqref{eqn:err-subg-norm}.
Similarly, the $\ell_2$-norm of the other perturbation $\|\eM_1 [h_t; u_t] - S_1 \sz_{t+1}\|$ is sub-Gaussian with its mean and sub-Gaussian norm bounded by  
\begin{align*}
    \bigo(H^{1/2} (H (d_y + d_u) + d_x)^{21/2} d_x^{3/2} T^{-1/2} \log^{1/2}(H/p)).
\end{align*}

Hence, by the perturbed linear regression bound (Lemma~\ref{lem:pert-lr}) with $\epsilon = 0$ therein, for $T$ greater than a constant polynomial in the problem parameters, we have
\begin{align*}
    \|\eS_0 - S S_1^{\top}\|_2
    =\;& \bigo(H^{1/2} (H(d_y + d_u) + d_x)^{21/2} d_x^{3/2} T^{-1/2} \log^{1/2}(H/p) \cdot d_x^{1/2} \log(T/p)) \\
    =\;& \bigo(H^{11} d_x^{2} (d_y + d_u)^{21/2} T^{-1/2} \log^{3/2}(T/p)),
\end{align*}
where we use $H (d_y + d_u) + d_x = \bigo(H (d_y + d_u))$ due to $d_h = H(d_y + d_u) \ge d_x$.  
As a result, 
\begin{align*}
    \|\eA - S \sA S^{\top}\|_2
    =\;& \|\eS_0 \tA - S S_1^{\top} S_1 \sA S^{\top}\|_2 \\
    =\;& \|(\eS_0 - S S_1^{\top}) \tA\|_2 + \| S S_1^{\top} (\tA - S_1 \sA S^{\top})\|_2 \\
    =\;& \bigo(H^{11} d_x^{2} (d_y + d_u)^{21/2} T^{-1/2} \log^{3/2}(T/p)),
\end{align*}
and $\|\eB - S \sB\|_2$ has the same order.
Hence, $\|\eA - S \sA S^{\top}\|_2$, $\|\eB - S\sB\|_2$ and $\|\eQ - S \sQ S^{\top}\|_2$ are all bounded by 
\begin{align*}
    \bigo(H^{11} d_x^{23/2} (d_y + d_u)^{21/2} T^{-1/2} \log^{5/2}(T/p)).
\end{align*}

\vspace{2pt}
\noindent\textbf{Certainty equivalent linear quadratic control.} 
As argued in Part I of this work, $\E[x_t^{\top} \sQ x_t] - \E[(\sz_t)^{\top} \sQ \sz_t] = \ipc{\sQ}{\E[(x_t - \sz_t)(x_t - \sz_t)^{\top}]}_{F}$ is a constant regardless of the actions $(u_\tau)_{\tau \le t}$, and it suffices to consider the latent state space for studying the policy suboptimality gap.
In the latent state space, for $t \ge H$, the action $u_t = \eK \eM h_t = \eK (S \sz_t + \delta_t^z)$, where $\delta_t^z := (\eM - S \sM) h_t - S \delta_t$ is a Gaussian noise vector correlated with $\sz_t$. Recall that $\delta_t = (\soA)^{H} \sz_{t-H}$ is the residual error.
Since $\sz_t = \sM h_t + \delta_t$, we have 
\begin{align*}
    \delta_t^z 
    = (\eM - S \sM) (\sM)^{\dagger} (\sz_t - \delta_t) - S \delta_t
    = (\eM - S \sM) (\sM)^{\dagger} \sz_t - \eM (\sM)^{\dagger} \delta_t.
\end{align*}

By Proposition~\ref{prp:full-rank-cov}, $\sigma_{\min}(\sM) = \Omega(\nu H^{-1/2}) = \Omega(H^{-1/2})$. Hence, we have 
\begin{align*}
    \|(\eM - S \sM) (\sM)^{\dagger}\|_2 = \bigo( H^{1/2} \|\eM - S \sM\|_2).
\end{align*}
Then, substituting $\delta_t^z$ into $u_t$, we have
\begin{align*}
    u_t = \eK (S + (\eM - S \sM) (\sM)^{\dagger}) \sz_t - \eK \eM (\sM)^{\dagger} \delta_t
    = \tK \sz_t - \eK \eM (\sM)^{\dagger} \delta_t,
\end{align*}
where $\tK := \eK (S + (\eM - S \sM) (\sM)^{\dagger})$.

We now consider the stability of the system $(\sA, \sB)$ under feedback gain $\tK S$.
By~\citep[Propositions 1 and 2]{mania2019certainty}, there exists a dimension-free constant $T_0 > 0$ depending polynomially on the problem parameters such that as long as $T \ge T_0$, $\|\eK - \sK S^{\top}\|_2$ is \mbox{bounded by}  
\begin{align*}
    \bigo(H^{11} d_x^{23/2} (d_y + d_u)^{21/2} T^{-1/2} \log^{5/2}(T/p)),
\end{align*}
and that $\eK$ stabilizes the system $(S\sA S^{\top}, S\sB)$, i.e., $\sA + \sB \eK S$ is stable.
This also implies that $\|\eK\|_2 = \bigo(\|\sK\|_2) = \bigo(1)$.
By~\citep[Lemma 5]{mania2019certainty}, there exists a dimension-free constant $\eps_0 > 0$ with $\eps_0^{-1}$ depending polynomially on the problem parameters, such that as long as $\|(\eM - S \sM) (\sM)^{\dagger}\|_2 \le \eps_0$, $\sA + \sB \tK$ is stable.

The control input $u_t = \tK \sz_t - \eK \eM (\sM)^{\dagger} \delta_t$, $t \ge H$ contains an additional \emph{perturbation term}: $-\eK \eM (\sM)^{\dagger} \delta_t$, whose covariance matrix is bounded by
\begin{align*}
    \|\eK \eM (\sM)^{\dagger} (\soA)^{H}\|_2 \|\Cov(\sz_{t-H})\|_2^{1/2}
    = \bigo(H^{1/2} \alpha \rho^{H} \|\Cov(\sz_{t-H})\|_2^{1/2}).
\end{align*}
For $H\le t \le 2H$, $\|\Cov(\sz_{t-H})\|_2 = \bigo(1)$ by an arbitrary stabilizing controller, as discussed at the end of~\S\ref{sec:setup}. 
Let $\Sigma$ denote the covariance matrix of the stationary distribution of $(\sz_t)_{t\geq 0}$ in the system $(\sA, \sB)$ under the controller $u_t = \tK \sz_t$, $t \ge 0$.
By Assumption~\ref{asp:asp}, $\|\Sigma\|_2 = \bigo(1)$. Hence, $ \|\Cov(\sz_{t-H}) - \Sigma\|_2$ is also of order $\bigo(1)$ for $H\le t\le 2H$.

Applying Lemma~\ref{lem:ls-pert} to the latent system $(\sA, \sB)$ for $t \ge H$, we have that for large enough $t$, 
\begin{align}  \label{eqn:pert-cov-diff}
    \|\Cov(\sz_t) - \Sigma\|_2 =\;& \bigo(H^{1/2} \alpha \rho^{H}),
\end{align}
where the $\eps$ in Lemma~\ref{lem:ls-pert} is on the order of~\eqref{eqn:pert-cov-diff} and satisfies the required conditions by our choice of $H$.
Hence, for large enough $t$, the cost difference incurred by the perturbation term at each step is bounded by
\begin{align*}
    \big|\ipc{\sQ}{\Cov(\sz_t) - \Sigma}_{F}\big|
    + \big|\ipc{\sR}{\Cov(\tK \sz_t - \eK \eM (\sM)^{\dagger} \delta_t) - \Cov(\tK \sz_t)}_{F}\big|
    = \bigo((d_x + d_u) H^{1/2} \alpha \rho^{H}),
\end{align*}
where the bound on covariance difference in the second term is due to~\citep[Lemma 14]{tian2022cost}. This implies that the difference in the time-averaged expected cost satisfies
\begin{align}  \label{eqn:e2e-1}
    |J(\epi) - J((\eM, \tK))| = \bigo((d_x + d_u) H^{1/2} \alpha \rho^{H}).
\end{align}

Finally, we consider the policy suboptimality gap
\begin{align*}
    J(\epi) - J(\spi) = J(\epi) - J((\eM, \tK)) + J((\eM, \tK)) - J(\spi).
\end{align*}
Let $\sP$ denote the optimal value matrix in system $(\sA, \sB, \sQ, \sR)$ given by the DARE~\eqref{eqn:p-riccati}.
By Assumption~\ref{asp:asp}, $\|\sP\|_2 = \bigo(1)$. 
By~\citep[Lemma 3]{mania2019certainty} (also~\citep[Lemma 12]{fazel2018global}), we have
\begin{align*}
    J((\eM, \tK)) - J(\spi)
    =\;& \bigo(\tr(\Sigma (\tK - \sK)^{\top} (\sR + (\sB)^{\top}\sP \sB) (\tK - \sK) )) \\
    =\;& \bigo(\|\Sigma\|_2 \|\sR + (\sB)^{\top}\sP \sB\|_2 \|\tK - \sK\|_F^2).
\end{align*}

Since 
\begin{align*}
    \|\tK - \sK\|_2
    \le\;& \|\eK\|_2 \|(\eM - S \sM) (\sM)^{\dagger}\|_2 + \|\eK - \sK S^{\top}\|_2 \\
    =\;& \bigo(H^{1/2} \|\eM - S \sM\|_2
    + H^{11} d_x^{23/2} (d_y + d_u)^{21/2} T^{-1/2} \log^{5/2}(T/p)) \\
    \overset{(i)}{=}\;& \bigo(H^{11} d_x^{23/2} (d_y + d_u)^{21/2} T^{-1/2} \log^{5/2}(T/p)),
\end{align*}
where $(i)$ is due to the bound on $\|\eM - S \sM\|_2$ in~\eqref{eqn:n-m-bd}, the suboptimality gap $J((\eM, \tK)) - J(\spi)$ is bounded by 
\begin{align}  \label{eqn:e2e-2}
    \bigo(H^{22} d_x^{23} (d_y + d_u)^{21} d_u T^{-1} \log^{5}(T/p)).
\end{align}
Due to our choice of $H$, the bound in~\eqref{eqn:e2e-2} dominates that in~\eqref{eqn:e2e-1}.
Hence, combining the bounds in~\eqref{eqn:e2e-1} and~\eqref{eqn:e2e-2}, we have
\begin{align*}
    J(\epi) - J(\spi) =\;& \bigo(H^{22} d_x^{23} (d_y + d_u)^{21} d_u T^{-1} \log^{5}(T/p)),
\end{align*}
which completes the proof.
\endproof

\section{Additional discussion on MuZero}

In this section, we discuss MuZero~\citep{schrittwieser2020mastering} in more detail, given its impressive performance and connection to this work.
Announced by DeepMind in 2019, MuZero extends the line of works including AlphaGo~\citep{silver2016mastering}, AlphaGo Zero~\citep{silver2017mastering}, and AlphaZero~\citep{silver2018general} by obviating the knowledge of the game rules.
MuZero matches the superhuman performance of AlphaZero in Go, shogi and chess, while outperforming model-free RL algorithms in Atari games. MuZero builds upon the powerful planning procedure of Monte Carlo Tree Search, with the major innovation being \emph{learning a latent model}. The latent model replaces the rule-based simulator during planning, and avoids the burdensome planning in pixel space for \mbox{Atari games}.

MuZero is a milestone algorithm in representation learning for control. Intuitively, the algorithm design makes sense, but its complexity has so far inhibited a formal theoretical study. 
On the other hand, statistical learning theory for linear dynamical systems and control has evolved rapidly in recent years~\citep{tsiamis2022statistical}; for partially observable linear dynamical systems, much of the work relies on learning \emph{Markov parameters}, lacking a direct connection to the empirical methods used in practice for possibly nonlinear systems. In this work, we aim to bridge the two areas by studying provable MuZero-style latent model learning in LQG control. In a sense, this work can be seen as a case study of the state representation learning algorithm of MuZero in linear systems.

The state representation learning algorithm of MuZero features three ingredients: 1) stacking frames, i.e., observations, as input to the representation function; 2) predicting costs, ``optimal'' values, and ``optimal'' actions from latent states; and 3) implicit learning of latent dynamics by predicting these quantities from latent states at future time steps. 
These are the defining characteristics of the MuZero-style algorithm that we shall consider.
In this work, we also handle partial observability by using a finite-length history, but we use a history of observations and actions, rather than only observations. 
In MuZero, the ``optimal'' values and actions are found by the powerful online planning procedure. In this work, we simplify the setup by considering data collected using random actions, which are known to suffice for identifying a partially observable linear dynamical system~\citep{oymak2019non}. In this setup, the values become those associated with this trivial policy and we do not predict actions since they are simply random noises. 
Note that although our study of the above ingredients is directly motivated by MuZero, previous empirical works have also explored them. 
For example, frame stacking has been a widely used technique to handle partial observability~\citep{mnih2013playing, mnih2015human, oh2017value}; predicting values for learning a latent model has been studied in~\citep{oh2017value}, which also learns the latent state transition implicitly.

\section{Concluding remarks}

We studied cost-driven state representation learning for solving unknown infinite-horizon time-invariant LQG control. 
We established finite-sample guarantees for two methods, which differ in whether the latent state dynamics is learned \emph{explicitly} by minimizing the transition prediction errors, or \emph{implicitly}  by using the transition for future cost predictions, with the latter being motivated by that used in MuZero~\citep{schrittwieser2020mastering}. 
For MuZero-style latent model learning, our analysis identified a coordinate misalignment problem in the latent state space, suggesting the value of \emph{multi-step} future cost prediction.
A limitation of this work is that we only considered state representation based on truncated histories, i.e., frame stacking, as used in MuZero; the \emph{recursive form} of the representation function, as in the Kalman filter, is also used  empirically~\citep{ha2018world, hafner2019dream}, and might be worth \mbox{further investigation}.

Together with Part I of this work, we have established a theoretical framework for analyzing cost-driven state representation learning for control. 
This opens up many opportunities for future research. 
For example, one may wonder about the extent to which cost-driven state representation learning provably generalizes to nonlinear observations or systems. 
Besides, one argument for favoring latent-model-based over model-free methods is their ability to generalize across different tasks; our framework may offer a perspective to formalize this intuition.
Moreover, given the ubiquity of visual perception in real-world control systems, it is of practical value to study state representation learning with a time-varying observation function or multiple observation functions, modeling images taken from different angles.

\section*{Acknowledgment} 
This work was supported in part by the NSF TRIPODS program (award number DMS-2022448). KZ acknowledges partial support from Simons-Berkeley Research Fellowship, the U.S. Army Research Office grant W911NF-24-1-0085, and the NSF CAREER Award 2443704.

\bibliographystyle{plainnat}
\bibliography{latent}